\documentclass[10pt,onecolumn]{article}
\usepackage{preamble}
\usepackage{makecell}
\crefname{equation}{eqn.}{eqns.}
\usepackage{multirow}
\usepackage{subcaption}
\usepackage[toc,page,header]{appendix}
\usepackage{minitoc}
\usepackage{wrapfig}

\newcommand{\ie}{\emph{i.e., }}
\newcommand{\eg}{\emph{e.g., }}

\newcommand{\jnt}{\textrm{joint}}
\newcommand{\lora}{\textrm{LoRA}}

\usepackage[round]{natbib}

\usepackage{zref-xr,zref-user}

\title{\LARGE Eigenfunction Extraction for Ordered Representation Learning}

\author{
\\
Burak Varıcı\thanks{Equal contribution.} \quad
Che-Ping Tsai\footnotemark[1] \quad
Ritabrata Ray \quad
Nicholas M. Boffi \quad
Pradeep Ravikumar \\
\\
Machine Learning Department, Carnegie Mellon University
}

\date{}

\begin{document}

\maketitle

\begin{abstract}
Recent advances in representation learning reveal that widely used objectives, such as contrastive and non-contrastive, implicitly perform spectral decomposition of a \emph{contextual kernel}, induced by the relationship between inputs and their \emph{contexts}. Yet, these methods recover only the linear span of top eigenfunctions of the kernel, whereas exact spectral decomposition is essential for understanding feature ordering and importance. In this work, we propose a general framework to extract \emph{ordered} and \emph{identifiable}  eigenfunctions, based on modular building blocks designed to satisfy key desiderata, including compatibility with the contextual kernel and scalability to modern settings. We then show how two main methodological paradigms, low-rank approximation and Rayleigh quotient optimization, align with this framework for eigenfunction extraction. Finally, we validate our approach on synthetic kernels and demonstrate on real-world image datasets that the recovered eigenvalues act as effective importance scores for feature selection, enabling principled efficiency–accuracy tradeoffs via adaptive-dimensional representations. \looseness=-1

\end{abstract}

\section{Introduction}\label{sec:intro}

Representation learning~\citep{bengio2012deep} underpins recent advances in artificial intelligence, powering large-scale systems such as information retrieval~\citep{karpukhin2020dense}, vector databases \citep{malkov2018efficient}, and retrieval-augmented generative models~\citep{lewis2020retrieval}. 
Despite this widespread and huge empirical progress, the question of identifiability, whether learned representations correspond to unique and interpretable structures, has often been overlooked~\citep{reizinger2025position}. This gap leaves a fragmented understanding of the mechanisms underlying representation learning.

Recent advances in representation learning theory~\citep{zhai2025contextures,johnson2023contrastive,balestriero2022contrastive} show that widely used objectives, such as contrastive~\citep{haochen2021provable} and non-contrastive~\citep{bardes2022vicreg}, implicitly recover the top-$d$ eigenspace of the \emph{contextual kernel} (see \Cref{sec:problem}). 
This spectral viewpoint is compelling: kernel methods have historically played a central role in machine learning, from support vector machines \citep{cortes1995support} to neural tangent kernels \citep{jacot2018neural}, and spectral decomposition of kernels has been the foundation of manifold learning \citep{tenenbaum2000global}, Laplacian eigenmaps \citep{belkin2003laplacian}, and diffusion maps \citep{coifman2006diffusion}. Yet, despite these connections, the implications of spectral decomposition for modern deep representation learning, particularly from an identifiability perspective, remain underexplored. \looseness=-1

By contrast, spectral methods like eigenvalue and singular value decompositions are indispensable across scientific domains, where extracting eigenfunctions and eigenvalues provides interpretable, fine-grained insights. Motivated by these needs, recent work has developed neural-network-based methods for eigenfunction extraction~\citep{pfau2019spectral,deng2022neural,ryu2024operator}. However, this line of work remains largely disconnected from representation learning: the former prioritizes explicit eigenfunction recovery, while the latter is driven by empirical performance and seldom exploits granular, eigenfunction-level information. 

This disconnection overlooks the key difference between eigenspaces and eigenfunctions. While existing representation learning methods recover only the top-$d$ \emph{eigenspace} (linear subspace spanned by the top-$d$ eigenfunctions), exact eigenfunctions and their associated eigenvalues offer a finer decomposition that assigns importance to individual features. Such identifiability is crucial in applications like feature selection and interpretability. \looseness=-1

To bridge this gap, we propose a general framework to extract identifiable and ordered eigenfunctions using modular building blocks designed to satisfy key desiderata, including compatibility with the contextual kernel and the scalability demands of modern applications. 
We then show two main methodological paradigms for eigendecomposition (low-rank approximation and Rayleigh quotient optimization) fit into this framework, and how they connect to contrastive and non-contrastive learning methods, enabling our modular framework to extract ordered features along with importance scores.

Finally, we validate our methods on synthetic kernels and real-world image datasets. On synthetic data, our algorithms recover the true eigenfunctions, confirming the theoretical guarantees. On image data, recovered eigenvalues provide effective importance scores for feature selection~\citep{kusupati2022matryoshka}. This enables \emph{adaptive-dimensional} representations: practitioners can scale down to reduce storage and compute or scale up for maximum accuracy. Having such representations supports deploying the same model across diverse environments, from resource-constrained devices to high-performance systems, while maintaining a principled efficiency-accuracy tradeoff.

\section{Problem Setting}\label{sec:problem}

\paragraph{Notation.} We adopt bra-ket notation for brevity: $\ket{\phi}$ denotes the function $\phi(\cdot)$, and $\braket{\phi}{\phi'} = \int \phi(x) \phi'(x) \mu(dx)$ is the inner product of $\mcX$ with some underlying measure $\mu$. The rank-one operator $\ketbra{\phi}{\psi}:\mcH_1 \to \mcH_2$ acts as $(\ketbra{\phi}{\psi})\ket{h} := \braket{\psi}{h}\,\ket{\phi}$---the outer product. 

\subsection{Spectral Decomposition of Operators}\label{sec:kernel-eigendecomposition}

Let $\mcH_1$ and $\mcH_2$ be separable Hilbert spaces and $T:\mcH_1 \to \mcH_2$ be a Hilbert-Schmidt (or more generally, compact) operator. By the Hilbert-Schmidt theorem, $T$ admits a \emph{singular value decomposition} (SVD): 
\begin{equation}\label{eq:svd-T-braket}
    T = \sum_{i=1}^{\infty} s_i \ketbra{\phi^*_i}{\psi^*_i} \ . 
\end{equation} 
Here, $(s_i)$ is a non-increasing sequence of non-negative \emph{singular values}, and $(\phi^*_i)$ and $(\psi^*_i)$ are corresponding left and right \emph{singular functions}, forming orthonormal bases for $\mcH_2$ and $\mcH_1$, respectively. These satisfy $T \psi^*_i = s_i \phi^*_i$ and $T^* \phi^*_i = s_i \psi^*_i$, where $T^*$ is the adjoint of $T$. Note that we assume \textbf{all singular values are \emph{distinct}} throughout the paper, ensuring a \emph{unique} decomposition and well-defined singular functions.

\paragraph{EVD Case.} 
An important special case is the self-adjoint operator, \ie $T^*=T$ and  $\mcH_1 = \mcH_2=\mcH$. In this case, SVD reduces to \emph{eigenvalue decomposition} (EVD): $T = \sum_{i=1}^{\infty} \lambda_i \ketbra{\phi^*_i}{\phi^*_i}$, where left and right singular functions coincide ($\phi^*_i=\psi^*_i$) and are called \emph{eigenfunctions}. For a positive (semi)definite operator $T = A^*A$, the eigenvalues equal the squared singular values of $A$, that is $\lambda_i = s_i^2$.
For brevity, \textbf{we refer to singular functions as eigenfunctions when unambiguous}, which will become clearer shortly.

\paragraph{Extracting Top Eigenfunctions.} With the EVD/SVD in place, the central question is \emph{how to extract the eigenfunctions of a compact operator}. In practice, we typically care about the \emph{top-$d$} eigenfunctions associated with the largest eigenvalues $(\lambda_i)_{i\in[d]}$. Numerous algorithms exist for matrix decomposition---variational characterizations, power iteration, factorizations---but only some extend clearly to \emph{infinite-dimensional} settings. Specifically,  inner products and norms extend naturally to Hilbert spaces, making variational formulations based on them a natural choice. 

Before introducing specific methods, we briefly recall \emph{why} extracting the top eigenfunctions is so valuable.  Given a fixed $d<n$, the eigenfunctions associated with the top-$d$ eigenvalues provide a close approximation of the original operator $T$. Formally, let $T_d$ denote the top-$d$ truncated decomposition as $T_d \coloneq \sum_{i=1}^{d} s_i \ketbra{\phi^*_i}{\psi^*_i}$.
By Schmidt's seminal theorem~\citet{schmidt1907theorie}, $T_d$ is the best approximation of $T$ among all operators of rank at most $d$. For additional intuition, the analogous statement for finite dimensions was discovered much later: \citet{eckart1936approximation} established it for the Frobenius norm, and \citet{mirsky1960symmetric} generalized it to all unitarily invariant norms.

\begin{theorem}\label{thm:SEYM} We have the following results for compact operators~\citep{schmidt1907theorie} and finite-dimensional matrices~\citep{eckart1936approximation,mirsky1960symmetric}:
    \begin{enumerate}
        \item Let $T, T' : \mcH_1 \to \mcH_2$ be compact operators where $\rank(T') = d \leq \rank(T)$. Then, we have $\norm{T - T_d} \le \norm{T - T'}$ for Hilbert-Schmidt norm.
        \item Let $\bM, \bM' \in \R^{m \times n}$ be positive semidefinite matrices where $\rank(\bM')=d \leq \rank(\bM)= r \leq \min\{m,n\}$. Let $\bM=\sum_{i=1}^{r} s_i \bu_i \bv_i^{\top}$ be the singular value decomposition of $\bM$ and define $\bM_d=\sum_{i=1}^{d}s_i \bu_i \bv_i^{\top}$. Then $\norm{\bM-\bM_d}\leq \norm{\bM-\bM'}$ for any unitarily invariant norm.
    \end{enumerate}
\end{theorem}
A norm is unitarily invariant if $\norm{U^\ast T V} = \norm{T}$ for all unitary $U : \mcH_2 \to \mcH_2$ and $V : \mcH_1 \to \mcH_1$. This family includes the common matrix norms such as Frobenius norm $\norm{T}_{\rm F}$ and the operator norm $\norm{T}_{\rm op}$, both special cases of Schatten $p$-norms $ \norm{T}_{S_p} \coloneq \left( \sum_{i=1}^{\infty} s_i^p \right)^{\frac{1}{p}}$, with $p=2$ and $p=\infty$, respectively. We discuss variational objectives for minimizing these norms in \Cref{sec:variational-results}.

\begin{remark}[Alternative objectives]\label{remark:d-metric}
   More generally, there could exist other functionals $\mcD(T,A)$ that are also minimized by $T_d$. We focus on the Hilbert-Schmidt and operator norms for their optimization tractability and relevance in prior work. Exploring broader classes of norms and functionals is an interesting direction for future work.
\end{remark}

\subsection{Application to Representation Learning}\label{sec:problem-representation-learning}

In representation learning, we typically learn representations of the input $x \in \mcX$ using its relationship with a \emph{context} variable $a \in \mcA$. We model this relationship with the joint distribution $\pplus(x,a)$. A prominent example is the CLIP model~\citep{CLIP}, where $\pplus(x,a)$ is the distribution of text-image pairs (see \citet[Sec. 2.1]{zhai2025contextures} for more examples). To formalize the relationship, we define the \textbf{contextual kernel} $\kxa: \mcX \times \mcA \to \Rnonn$ as the ratio of the joint distribution to the product of the marginal distributions of inputs, $\px$, and their contexts, $\pa$:
\begin{equation}\label{eq:kxa}  
    \kxa(x,a) = \frac{\pplus(x,a)}{\px(x) \pa(a)}   = \frac{\pplus(a|x)}{\pa(a)} = \frac{\pplus(x|a)}{\px(x)} \ . 
\end{equation}
This kernel induces an integral operator $\txa : \lap \to \lxp$ between the corresponding $L^2$ spaces over $\pa$ and $\px$. For any function $g \in \lap$, this operator simplifies to the conditional expectation:
\begin{equation}\label{eq:txa}
    \begin{aligned}
        (\txa g)(x) &= \int g(a) \kxa(x,a) \pa(a) da = \int g(a) \pplus(a|x) da = \E_A[g(a) | x] .
    \end{aligned}
\end{equation}
Crucially, computing $\txa$ only requires sampling $a \sim \pplus(\cdot | x)$, not explicit access to $\pplus(x,a)$. The spectral decomposition of the contextual kernel is given below.
\begin{lemma}\citep[Proposition 1]{zhai2024understanding}\label{lm:spectral-decomposition}
    Assume that $\txa$ is a compact operator. The spectral decomposition of $\kxa$ is given by $\kxa(x,a) = \frac{\pplus(x,a)}{\px(x) \pa(a)} = \sum_{i=1}^{\infty} s_i \phi^*_i(x) \psi^*_i(a)$, where $\phi_1 \equiv 1$ and $\psi_1 \equiv 1$ are constant functions with singular value $s_1=1$.
\end{lemma}
In practice, we use two $d$-dimensional encoders, $\Phi := [\phi_1,\dots,\phi_d]: \mcX \mapsto \R^d$ for inputs and $\Psi:= [\psi_1,\dots,\psi_d]: \mcA \mapsto \R^d$ for their contexts. These encoders are trained to approximate the contextual kernel with their inner product: $\kxa(x,a) \approx \Phi(x)^\top \Psi(a) = \sum_{i=1}^d \phi_i(x)\psi_i(a)$.

\paragraph{EVD Case.}
A common special case involves learning a \emph{single} encoder for one space. For instance, in vision self-supervised learning, context $a$ is an augmentation of the input image $x$, and a single encoder is used~\citep{chen2020simple}. This setting is modeled by the \emph{positive-pair kernel} $\kaa : \mcA \times \mcA \to \Rnonn$~\citep{johnson2023contrastive}: 
\begin{equation}
\kaa(a,a') = \frac{\pplus(a,a')}{\pa(a) \pa(a')} = \frac{\int \pplus(a|x) \pplus(a'|x) d \px(x) }  {\pa(a) \pa(a')},
\end{equation}
which captures the probability of context instances $(a,a')$ sampled conditioned on the same input. The operator $\taa$ of $\kaa$ is self-adjoint and given by $\taa = \tax \txa$, where $\tax$ is the adjoint of $\txa$. Therefore, EVD of $\kaa$ is $ \kaa(a,a') =  \sum_{i=1}^{\infty} s^2_i \psi^*_i(a) \psi^*_i(a')$, and single encoder $\Psi$ suffices to approximate the kernel, \ie $\kaa(a,a') \approx \Psi(a)^\top\Psi(a') = \sum_{i=1}^d \psi_i(a)\psi_i(a')$. Dual kernel $\kxx$ can be defined similarly~\citep{zhai2024understanding}. Note that eigenfunctions of $\kxx$ and $\kaa$ are left and right singular functions of $\kxa$, respectively.

\boxedtextrelaxed{\textbf{Objective.} Our goal is to extract the top-$d$ eigenpairs of the contextual kernel $\kxa$ (or $\txa$) induced by the interaction between $X$ and $A$.}

\subsection{Desiderata}\label{sec:desiderata}

This paper provides a unified perspective on \emph{eigenfunction extraction for representation learning, aiming to establish a foundation for identifiable and ordered representations.} The study of eigenfunctions, long rooted in kernel methods~\citep{scholkopf2002learning}, has recently shifted to neural approximations to meet modern scalability demands. Despite growing interest, findings remain fragmented and often lack a clear machine learning focus; crucial distinctions (\eg learning \emph{exact} eigenfunctions versus only their linear span) are underexplored. To this end, we first formalize key desiderata. 

\subsubsection{Compatibility with Contextual Kernel}\label{sec:desiderata-contextual-kernel}

Our perspective links representation learning and eigenfunction extraction, focusing on methods that learn the eigenfunctions of the contextual kernel $\kxa$.  In machine learning practice, we have only \emph{sample access} to the joint $P^+(x,a)$, \eg drawing $x \sim \px$ and $a \sim P^+(\cdot | x)$. Algorithms must therefore recover the top-$d$ eigenfunctions of $\kxa$ given only these samples.

Neural approaches are well-suited for this task as stochastic mini-batch training overcomes the scalability limits of classical methods (\eg $\mathcal{O}(n^3)$ for kernel PCA), thereby scaling to large, high-dimensional data. A recurring challenge, however, is optimization instability arising from biased gradient estimates, where post-hoc remedies can introduce issues~\citep{pfau2019spectral}. Thus, a method should be inherently compatible with stable, theoretically sound stochastic optimization. 

\boxedtextrelaxed{\textbf{Desideratum:} Extract eigenfunctions using only sample access to $P^+(x,a)$, with efficient mini-batch processing, and theoretically sound optimization.}

\subsubsection{Exact Eigendecomposition}\label{sec:desiderata-eigendecomposition}

Recovering the top-$d$ \textbf{eigenspace} (the linear span of the top-$d$ eigenfunctions) is a simpler problem than recovering the exact eigenfunctions themselves. We say that an encoder $\Phi = [\phi_1,\cdots,\phi_d]$ \textbf{extracts the eigenspace} of $\kxa$, if $\sspan \oset{\phi_1,\cdots,\phi_d} =  \sspan \oset{\phi^*_1,\cdots,\phi^*_d}$ where $\oset{\phi^*_1,\cdots,\phi^*_d}$ are the true top-$d$ eigenfunctions. 

While learning the eigenspace may suffice for some applications (\ie when a linear probe atop $\Phi$ is good enough), an exact eigendecomposition yields \emph{ordered} eigenfunctions and their eigenvalues. These eigenvalues can serve as importance scores and enable adaptive-dimensional representations, which is increasingly valuable in billion-scale search settings, \eg for efficient inference as shown in Matryoshka representation learning~\citep{kusupati2022matryoshka}. 

Our goal is to develop a general framework capable of extracting these precise eigenfunctions and eigenvalues. 
We say that encoder $\Phi$ \textbf{extracts the ordered eigenfunctions} of $\kxa$ if $\phi_i \equiv \pm \phi^*_i$ for all $i \in [d]$, where $(\phi^*_i)_{i\in[d]}$ are top-$d$ left singular functions of $\txa$. This also implies recovering the singular values as $\norm{\tax \phi_i}_\px = |s_i | \norm{\psi^*_i}_\pa = s_i$. 
The following simple example illustrates the importance of knowing the exact eigenfunctions.
\begin{example}\label{example:ridge}
    Consider a ridge regression model built on encoder $\Phi$: $\min_{\bw,b} \; \E \big[\paren{\bw^{\top} \Phi(X) + b - Y}^2\big]  + \beta \norm{\bw}_2^2$.  This is suboptimal because its uniform penalty treats all features equally, whereas features from less important eigenfunctions (with smaller eigenvalues) should be penalized more heavily. Learning only the eigenspace prevents such principled regularization. In contrast, knowing the exact eigenfunctions enables constructing an optimally weighted encoder, $\tdPhi = [\sqrt{\lambda_1} \phi^*_1, \cdots, \sqrt{\lambda_d} \phi^*_d]$~\citep{zhai2024stkr}. 
\end{example}
\boxedtextrelaxed{\textbf{Desideratum:} An algorithm must extract the \emph{eigenfunctions}, not just their collective \emph{eigenspace}.}

\subsubsection{Unconstrained Optimization}\label{sec:desiderata-unconstrained}

Some eigenfunction extraction methods optimize under orthonormality constraints. While methods like projected gradient descent can directly handle constrained optimization, implementing the projection step is challenging for neural networks with millions of parameters.

\boxedtextrelaxed{\textbf{Desideratum:} The learning objective should be an \emph{unconstrained} optimization problem.}

\subsubsection{Computational Efficiency}\label{sec:desiderata-efficiency}

A core premise of neural approaches is their scalability. Hence, computational efficiency is a key desideratum, with two primary components.
\begin{description}[leftmargin=1em]
    \item[Joint Optimization.] Sequentially learning eigenfunctions---finding $\phi^*_1$, then using it to find $\phi^*_2$, and so on---is highly inefficient for large $d$ and can yield suboptimal solutions. Instead, one should \emph{jointly learn} all $d$ eigenfunctions in a single process. 
    \item[Joint Parameterization.]  Optimizing eigenfunctions jointly does not specify their parameterization, \eg some methods use separate neural networks for each eigenfunction~\citep{deng2022neural}, which is atypical in machine learning. Instead, \emph{parameter sharing} (\eg a single network with $d$ heads) is essential for efficiency.
\end{description}
\boxedtextrelaxed{\textbf{Desideratum:} An algorithm must support \emph{joint learning and parameterization} of the eigenfunctions.}

\section{Eigenfunction Extraction Framework}\label{sec:framework}

In this section, we build the framework for satisfying the desiderata outlined in \Cref{sec:desiderata}. We present a set of general methods (building blocks) for constructing eigenfunction extraction algorithms, where each block is designed to fulfill one or more of those desiderata. We begin by defining a generic optimization problem.
\begin{definition}[Base Optimization Problem]\label{def:base-optimization}
    Let $\mcO_1,\ldots,\mcO_d$ be a sequence of minimization problems, where each $\mcO_j$ has an objective of the form:
    \begin{equation}\label{eq:base-optimization-problem}
    \min_{(\phi_1,\ldots,\phi_j) \in \mcC_j} \mcL_j(\phi_1,\ldots,\phi_j) \ , 
    \end{equation}
    where $\mcC_j$ is the feasible set and each function $\phi_i$ lies in a shared function space $\mcF$, \eg $\lxp$. 
\end{definition}

\subsection{From Eigenspace to Eigenfunctions}\label{sec:eigenspace-to-eigenfunctions}

Extracting the exact eigenfunctions is the main goal, but learning the eigenspace is typically easier. Consequently, exact-extraction methods build on \emph{base} methods that recover the eigenspace. We formalize such a base method as an \emph{eigenspace extractor}, and outline three generic techniques to obtain exact eigenfunctions from it: two via a joint training principle, \emph{nesting}, and one via a post-hoc transformation, \emph{Rayleigh–Ritz}. 
\begin{definition}[Eigenspace Extractor]
    A base optimization problem $\mcO$ is an \textbf{eigenspace extractor} for a compact self-adjoint operator $T: \mcH \to \mcH$ if, for each $j \in [d]$, every minimizer $(\phi_1,\dots,\phi_j)$ of $\mcO_j$ spans the top-$j$ eigenspace of $T$:
        \begin{equation}
            \sspan(\phi_1,\ldots,\phi_j) = \sspan(\phi^*_1,\ldots,\phi^*_j) \ ,
        \end{equation}
        where $\phi^*_i$ is the $i^\textrm{th}$ eigenfunction of $T$.\footnote{To streamline exposition, we use the terminology of eigenfunctions; the definitions readily extend to singular functions by replacing $(\phi_j)_{j\in[d]}$ with pairs $\{(\phi_j,\psi_j)\}_{j \in [d]}$ in the base problem. For a non-self-adjoint $T$, an eigenspace extractor means $\sspan(\phi_1,\dots,\phi_j)=\sspan(\phi^*_1,\dots,\phi^*_j)$ and $\sspan(\psi_1,\dots,\psi_j)=\sspan(\psi^*_1,\dots,\psi^*_j)$.}
\end{definition}
To extract \emph{individual} eigenfunctions from an eigenspace extractor, we first formalize nesting.
\begin{definition}[Orthogonal Nested Minimizers]\label{def:nested-minimizer}
    A sequence of base problems $\mcO_1,\ldots,\mcO_d$ is said to have \emph{orthogonal nested minimizers} if (i) for all minimizers $(\hat \phi_1, \cdots \hat \phi_{j-1}) \in \argmin \mcO_{j-1}$, there exists $\hat \phi_j \not \equiv 0$ such that $(\hat \phi_1, \cdots \hat \phi_j) \in \argmin \mcO_{j}$, and (ii) all such $\hat \phi_j$ are orthogonal to preceeding functions, \ie $\braket{\hat\phi_j}{\hat\phi_i}_{\mcH} = 0$ for all $1 \leq i < j$.
\end{definition}
This condition ensures identifiability of eigenfunctions under nested eigenspace extraction: we can always find a new base function orthogonal to all previously extracted ones, as we will show shortly.

\subsubsection{Sequential Nesting}\label{sec:sequential-nesting}

The first approach \emph{sequential nesting} builds on the insight that a sequence of eigenspace extractors $\mcO_1,\dots,\mcO_d$ with \textbf{orthogonal} nested minimizers identifies the individual eigenfunctions.
If $\mcO_1$ admits a solution $\hat \phi_1$, then it must coincide with the first eigenfunction up to a scaling factor, \ie $\hat \phi_1 = c_1 \phi^*_1$. Then by \Cref{def:nested-minimizer}, there exists an orthogonal base function $\hat \phi_2$ satisfying $(\hat \phi_1, \hat \phi_2) \in \argmin_{\phi_1,\phi_2} \mcO_2( \phi_1,\phi_2)$, which can be found by solving $\hat \phi_2 \in \argmin_{\phi_2} \mcO_2(\hat \phi_1,\phi_2)$. Since $\sspan(\hat \phi_1,\hat \phi_2) = \sspan(\phi^*_1,\phi^*_2)$, it follows that $\hat \phi_2 = c_2 \phi^*_2$. 
This leads to a sequential optimization process where each new eigenfunction is found by solving $ \hat \phi_j \in \argmin_{\phi_j \in \mcF} \mcO_j(\hat \phi_1,\ldots,\hat \phi_{j-1},\phi_j)$.

\begin{theorem}(Proof in \Cref{app:proof-sequential-nesting})\label{thm:sequential-nesting-general}
    Assume that the base problems $\mcO_1,\dots,\mcO_d$ (i) are eigenspace extractors for a compact operator $T$, and (ii) admit orthogonal nested minimizers. Then, sequentially solving 
    \begin{equation}
        \tilde\mcO_j : \min_{\phi_j } \mcL_j\big( \hat \phi_1, \ldots, \hat \phi_{j-1}, \phi_j \big) \ , 
    \end{equation} 
    such that $(\hat \phi_1, \ldots, \hat \phi_{j-1}, \phi_j) \in \mcC_j$ and where $\hat \phi_i$ is the solution from step $i < j$, recovers the $j$-th eigenfunction of $T$ up to a scaling factor $c_j$, \ie $\hat \phi_j = c_j \phi^*_j$. 
\end{theorem}

\begin{remark}
    The term ``sequential nesting'' was introduced by~\cite{ryu2024operator}, but the concept builds on sequential methods in prior work~\citep{gemp2021eigengame,bengio2004learning,deng2022neural}. \Cref{thm:sequential-nesting-general} generalizes these approaches, casting them as special cases. While sequential nesting violates the desideratum of joint training \emph{and} parameterization, it serves as a crucial starting point. \Cref{sec:sequential-to-joint} outlines a standard lift to joint optimization. \end{remark}

\subsubsection{Joint Nesting}\label{sec:joint-nesting-general}

The second approach, \emph{joint nesting}, learns all top-$d$ eigenfunctions simultaneously. The key idea is that true eigenfunctions can be characterized as global minimizers of a single objective function, formed by a weighted sum of the base objectives $(\mcO_j)_{j \in [d]}$. 

The closest prior work by \cite{ryu2024operator} sums low-rank approximation objectives $(\mcO_j)_{j \in [d]}$ to jointly learn the top-$d$ eigenfunctions. We show that this is a special case of general joint nesting,  which extracts eigenfunctions from \emph{any} eigenspace extractor.
Formally, for positive coefficients $w_1,\dots,w_d$, define the \textbf{joint nested optimization} problem $\mcO_{\jnt}$:
\begin{equation}\label{eq:def-joint-nesting}
   \minimize_{(\phi_1,\ldots, \phi_d) \in \mcC_{[d]}} \mcL_{\jnt} \coloneq \sum_{i=1}^d w_i \mcL_i(\phi_1,\ldots,\phi_i)\ ,
\end{equation}
where $\mcC_{[d]}:= \bigcap_{i \in [d]}\mcC_i$ is intersection of feasible sets. 

\begin{theorem}(Proof in \Cref{app:proof-joint-nesting})\label{thm:joint-nesting-general} 
    If $\mcO_1,\dots,\mcO_d$ are eigenspace extractors for $T$ with orthogonal nested minimizers, then for any positive weights $(w_i)_{i\in[d]}$, solving $\mcO_{\jnt}$ recovers the eigenfunctions up to scaling, \ie its minimizers are of the form $(c_1 \phi^*_1,\ldots, c_d\phi^*_d)$, where $(\phi^*_i)_{i\in[d]}$ are the top-$d$ eigenfunctions of $T$. 
\end{theorem} 

\begin{remark}\label{remark:joint-nesting}
    Jointly optimizing nested objectives is a broadly used principle. Matryoshka representation learning~\citep{kusupati2022matryoshka} sums objectives across embedding lengths, a technique later used to produce efficient variable-length embeddings~\citep{openai2024embedding,vera2025embeddinggemma}. Additional applications of adaptive representations are discussed in \Cref{app:related_work_adaptive_representations}. 
\end{remark}

\subsubsection{Rayleigh-Ritz Method}\label{sec:rayleigh-ritz}

The third technique generalizes the classic Rayleigh-Ritz method for computing eigenvectors~\citep{Ritz+1909+1+61,PhysRev.43.830,trefethen1997numerical}. In the kernel setting, it serves as a post-processing step: given any orthonormal basis of the top eigenspace, project $T$ onto that subspace to obtain a finite matrix problem. The eigenvectors of this matrix then specify the correct linear combinations of basis functions to form the true eigenfunctions, as stated formally below.

\begin{theorem}\label{thm:rayleigh-ritz}(Proof in \Cref{app:proof-rayleigh-ritz})
    Suppose that $(\phi_i)_{i \in [d]}$ form an orthonormal basis for the top-$d$ eigenspace of a self-adjoint compact operator $T$. Let $\bB \in \R^{d \times d}$ be a matrix with entries $\bB_{ij} \vcentcolon = \braket{\phi_i}{T \phi_j}$. Denote the eigenpairs of $\bB$ by $\{(\lambda_i,\by_i) : i \in [d]\}$, where $(\lambda_i)_{i \in [d]}$ are in non-increasing order and $\{\by_i \in \R^d : i \in [d]\}$ are the corresponding eigenvectors. Then, the top $i^\textrm{th}$ eigenfunction of $T$ is given by $\phi^*_i = \sum_{j=1}^d (\by_i)_{j} \phi_j$, with the corresponding eigenvalue $\lambda^*_i=\lambda_i$.
\end{theorem}

The key advantage of the Rayleigh-Ritz method over the joint-nesting is that it circumvents the need to retrain the eigenspace extractor, as it operates through a post-hoc EVD procedure. This makes it especially useful when working with state-of-the-art embedding models that are not open-source, since eigenpairs can still be estimated by applying the Rayleigh-Ritz method directly through model inference.  Detailed algorithms are provided in \Cref{app:rayleigh_ritz_implementation}.

\subsection{Sequential to Joint Optimization}\label{sec:sequential-to-joint}

A sequential optimization process (as in sequential nesting) can be converted to a single end-to-end trainable objective. The key is to form a single loss function by summing the sequential objectives while strategically \emph{blocking gradients} to enforce the sequential dependency. Deep learning frameworks implement this with a stop-gradient ${\rm sg}(\cdot)$ operation, that forwards its input but treats it as constant during backpropagation, thereby \emph{freezing} the parameters that produced it~\citep{deng2022neural,pfau2019spectral}. The following general theorem formalizes this technique.
\begin{theorem}\label{thm:sequential-to-joint}
    Let $(\hat \phi_{\theta_1}, \ldots, \hat \phi_{\theta_d})$ be an orthogonal nested minimizer of unconstrained objectives $\mcL_{i}(\phi_{\theta_1},\ldots,\phi_{\theta_i})$. This minimizer can be found by solving the joint optimization problem 
    \begin{equation}
        \min_{\theta_1,\dots,\theta_d} \sum_{i=1}^d \mcL_i \big({\rm sg}(\phi_{\theta_1}), \dots, {\rm sg}(\phi_{\theta_{i-1}}), \phi_{\theta_i} \big) \ .
    \end{equation}
\end{theorem}
\begin{proof}
    The proof follows from the analysis of the gradient flow. The operator ${\rm sg}(\cdot)$ ensures that the loss term $\mcL_i$ does not provide a gradient to $\theta_j$ for $j<i$. Assuming that the parameters $(\theta_1,\dots,\theta_k)$ are disjoint, the gradient for $\theta_j$ is calculated solely from the term $\mcL_j$. This effectively decouples the optimization into the desired sequential structure, where each $\theta_i$ is optimized with the preceding solutions $\phi_{\theta_1},\dots,\phi_{\theta_{i-1}}$ treated as fixed inputs. Therefore, the theorem statement follows from \Cref{thm:sequential-nesting-general}.
\end{proof}

While this technique enables joint optimization, a significant drawback is the need for \emph{disjoint parameters} per eigenfunction, violating joint parameterization. Since parameter sharing is crucial for efficiency, methods like joint nesting are typically more suitable.

\subsection{Unconstrained Reformulation}\label{sec:constrained-to-unconstrained}

Constrained problems, such as $\min_{\Phi} \mcL(\Phi)$ subject to constraints $\{\mcC_j(\Phi) \leq 0 : j \in [m]\}$, are often converted to unconstrained ones via a Lagrangian~\citep[Chapter 5]{boyd2014convex}. This involves introducing positive \emph{Lagrangian multipliers} $\mu_j$ and penalty functions $\rho_j$ to create a new objective with penalty terms: $\mcL_{\rm unconstr.} = \mcL(\Phi) + \sum_{j=1}^{m} \mu_j \rho_j(\mcC_j(\Phi))$. In practice, the multipliers $(\mu_j)_{j \in [m]}$ become tunable positive hyperparameters that control the violations of each constraint.

\paragraph{Sample Splitting.} A common challenge arises when the constraints are defined at the population level, \ie 
$\mcC(\Phi) = \E_X \mcL_{\mcC}(\Phi(X))$, where $\mcC$ (\eg variance) is expressed as an expectation of losses over samples. For the variational objective to remain tractable, we require $\rho(\mcC(\Phi))$ to be decomposable across samples. One practical choice is to set $\rho$ as the squared error, which allows for unbiased estimation of the penalty functions. Specifically, we split each batch into two independent subsets and estimate the expectations separately; the independence of the two subsets gives 
\begin{equation}
    \big(\E_X \mcL_{\mcC}(\Phi(X))\big)^2 
    = \E_X \mcL_{\mcC}(\Phi(X)) \cdot \E_{X'} \mcL_{\mcC}(\Phi(X')) \ . 
\end{equation}

\paragraph{Bilevel Optimization.} One principled approach for addressing the problem of biased gradient estimates is to rewrite the problem as a bilevel optimization. The inner level estimates the expectation $ \E_X [\mcL_{\mcC}(\Phi(X))]$ over a batch, and the outer level applies $\rho$ to this estimate. The Spectral Inference Networks (SpIN) framework~\citep{pfau2019spectral} uses this strategy for learning eigenfunctions.

\section{Variational Objectives}\label{sec:variational-results}

With the general framework in hand, this section analyzes two objective categories for eigenfunction extraction: Low-Rank Approximation (LoRA) and Rayleigh Quotient (RQ). We show how both methods fit our framework for extracting the eigenspace and recovering ordered eigenfunctions, connect to contrastive and non-contrastive learning, and enable our modular approach to extract ordered features with importance scores. A broader discussion of relevant work appears in \Cref{app:related-work}. 

\subsection{Low-rank Approximation Methods}\label{sec:lora}

LoRA provides an \emph{unconstrained} route to eigenspace extraction. It builds on the observation that the best rank-$d$ approximation of a continuous kernel $K$ comes from truncating its Mercer expansion after the first $d$ eigenpairs, and the eigenfunctions can be recovered by minimizing the squared approximation error. This idea goes back to~\cite{bengio2004learning}, who proposed a sequential method resembling nesting.

\subsubsection{Eigenspace Extraction via LoRA}\label{sec:lora-eigenspace-extractor}

LoRA objectives can be used to learn the \emph{singular} functions of a compact linear operator $T: \mcH_1 \to \mcH_2$. Recall that the SVD of $T$ is $T = \sum_{i=1}^{\infty} s_i \ketbra{\phi^*_i}{\psi^*_i}$. The base LoRA objective is given by minimizing the Hilbert-Schmidt norm (Frobenius norm in finite-dimensional Euclidean space) of the approximation error: 
\begin{equation}\label{eq:lora-base}
    \minimize_{\psi_i \in \mcH_1, \ \phi_i \in \mcH_2} \norm[\Big]{T - \sum_{i=1}^{d} \ketbra{\phi_i}{\psi_i}}_{\rm HS}^2 \ .
\end{equation}
As stated in \Cref{thm:SEYM}, the solution of this is given by the truncated SVD of $T$~\citep{schmidt1907theorie}, that is, a minimizer $(\hat\Phi, \hat\Psi)$ of \eqref{eq:lora-base} satisfies $\sum_{i=1}^{d} \ketbra{\hat\phi_i}{\hat\psi_i} = \sum_{i=1}^{d} s_i \ketbra{\phi^*_i}{\psi^*_i}$. This classical result establishes LoRA as an ideal building block for our framework.
\begin{theorem}(Proof in \Cref{app:proof-lora-eigenspace-extractor})\label{thm:lora-eigenspace-extractor}
    The LoRA objective in \Cref{eq:lora-base} is an eigenspace extractor for $T$, \ie $\sspan(\hat\phi_1,\dots,\hat\phi_d) = \sspan(\phi^*_1,\dots,\phi^*_d)$ and $\sspan(\hat\psi_1,\dots,\hat\psi_d) = \sspan(\psi^*_1,\dots,\psi^*_d)$. Furthermore, it admits orthogonal nested minimizers, such that $\hat \phi_i = a_i \phi^*_i$ and $\hat \psi_i = b_i \psi^*_i$ for some nonzero scalars $a_i, b_i$ where $a_i b_i = s_i$. \looseness=-1
\end{theorem}
\noindent\textit{Proof Sketch.} We first prove the eigenspace property using the fact that the ranges of the solution operator $L_d \coloneq \sum_{i=1}^{d} \ketbra{\hat\phi_i}{\hat\psi_i}$ and the truncated operator $T_d \coloneq \sum_{i=1}^{d} s_i \ketbra{\phi^*_i}{\psi^*_i}$ must be identical. Then, we prove the orthogonal nested minimizer property by induction: assuming the first $d-1$ components are correct, the problem reduces to finding the best rank-$1$ approximation of the residual operator, $T-T_{d-1}$, whose unique solution yields the next singular function pair. 

Importantly, \Cref{thm:lora-eigenspace-extractor} provides a direct pathway to satisfying our desiderata. Since LoRA is an unconstrained eigenspace extractor with orthogonal nested minimizers, any of the sequential nesting (\Cref{thm:sequential-nesting-general}), joint nesting (\Cref{thm:joint-nesting-general}), or Rayleigh-Ritz method (\Cref{thm:rayleigh-ritz}) can be used to extract exact ordered singular functions in conjunction with the base LoRA objective.

For gradient-based optimization, we need a more convenient form of the LoRA objective in~\Cref{eq:lora-base}, which we can derive by expanding the squared Hilbert-Schmidt norm. 
Following \cite[Lemma 3.1]{ryu2024operator}, the approximation error decomposes into terms involving $T$ and inner products of the candidate functions, and subsequently, we define the \emph{low-rank approximation} (LoRA) objective $\mcL_d \coloneq \mcL_{\lora}(\Phi, \Psi)$ as 
\begin{equation}\label{eq:lora}
    \mcL_d \coloneq -2\sum_{i=1}^{d} \braket{\phi_i}{T \psi_i}_{\mcH_2} + \sum_{i=1}^{d} \sum_{j=1}^{d} \braket{\phi_i}{\phi_j}_{\mcH_2} \braket{\psi_i}{\psi_j}_{\mcH_1} \ .
\end{equation}
Minimizing $\mcL_d$ is equivalent to minimizing the original objective in \Cref{eq:lora-base}, as the two objectives only differ by the constant $\norm[\big]{T}_{\rm HS}^2$ (proof in \Cref{app:proof-lora-error-HS}). Therefore, $\mcL_d$ is also an eigenspace extractor with orthogonal nested minimizers, and it can be used in conjunction with the methods in \Cref{sec:eigenspace-to-eigenfunctions} to extract the \emph{ordered} top-$d$ singular functions.

\subsubsection{Optimization for Contextual Kernel}\label{sec:lora-contextual-kernel}

Next, we specialize the generic LoRA objective to the contextual kernel $\kxa = \frac{P^+(x,a)}{\px(x)\pa(a)}$. In this case, we have $T=\txa$, $\mcH_1 = \lap$, $\mcH_2 = \lxp$, and encoders $\Phi : \mcX \to \R^d$ and $\Psi : \mcA \to \R^d$. Substituting $\kxa$, the first term of $\mcL_d$ becomes an expectation over the joint $P^+(x,a)$:
\begin{equation}
    \sum_{i=1}^{d} \braket{\phi_i}{ \txa \psi_i} = \E_{(x,a) \sim P^+}\left[ \Phi(x)^{\top} \Psi(a) \right] \ .
\end{equation}
By expanding the inner products, the second term becomes an expectation over the product of marginals:
\begin{equation}
    \sum_{i=1}^{d} \sum_{j=1}^{d} \braket{\phi_i}{\phi_j} \braket{\psi_i}{\psi_j} = \E_{x \sim \px, a \sim \pa} \big[ \big(\Phi(x)^\top \Psi(a) \big)^2 \big] \ .
\end{equation}
Combining these yields the final variational objective:
\begin{equation}\label{eq:lora-contextual-final}
    \begin{aligned}
        \mcL_{\lora}(\Phi, \Psi) &= -2 \E_{x,a \sim P(x,a)}[ \Phi(x)^\top \Psi(a)] + \E_{x \sim \px, a \sim \pa} \big[ \big(\Phi(x)^\top \Psi(a) \big)^2 \big]  \ , 
    \end{aligned}
\end{equation}
which we can minimize using stochastic gradient descent. The full derivation is provided in \Cref{app:lora-contextual-kernel}.

\paragraph{Connection to contrastive learning.} Contrastive learning is a dominant self-supervised learning paradigm. A common objective, the Spectral Contrastive Loss (SCL)~\citep{haochen2021provable} is defined as 
\begin{equation}\label{eq:spectral-contrastive-loss}
    \begin{aligned}
        \mcL_{\rm SCL}:= - \E_{x \sim \px} &\E_{a, a^+ \sim P^+(a|x)}[\Psi(a)^\top \Psi(a^+)] + \frac{1}{2}\E_{a, a^{-} \sim \pa} \big[ \big(\Psi(a)^\top \Psi(a^{-}) \big)^2 \big]  \ , 
    \end{aligned}
\end{equation}
where $(a,a^+)$ are positive pairs (from the same $x$) and $(a,a^{-})$ are negative pairs. This is exactly the LoRA objective in \eqref{eq:lora} for $T = \taa$. Thus, the well-known result that SCL extracts the eigenspace of $\taa$~\citep{zhai2025contextures,johnson2023contrastive} follows directly from \Cref{thm:lora-eigenspace-extractor}, since SCL is a special case of $\mcL_{\lora}$. More broadly, the framework extends contrastive learning to non-self-adjoint operators for extracting subspaces of both left and right singular functions. 

\subsection{Rayleigh Quotient Optimization}\label{sec:rayleigh}

The second approach for neural eigenfunction extraction is based on the Rayleigh quotient (RQ)~\citep{horn2012matrix}. Classically used for finding the top eigenpairs of a matrix, the RQ principle also extends to eigenfunction extraction.
Specifically, for a compact linear operator $T: \mcH_1 \to \mcH_2$, we seek orthonormal functions that maximize the summed Rayleigh quotients $\braket{\phi_i}{T \psi_i}_{\mcH_2}$, \ie
\begin{equation}\label{eq:rq-svd}
\min_{\substack{\psi_i \in \mcH_1 \\ \phi_i \in \mcH_2}}
    -\sum_{i=1}^{d} \braket{\phi_i}{ T  \psi_i}_{\mcH_2}
    \;\;\st\;\;
    \begin{aligned}
        \braket{\phi_i}{\phi_j}_{\mcH_2} &= \delta_{ij} \\
        \braket{\psi_i}{\psi_j}_{\mcH_1} &= \delta_{ij}
    \end{aligned} \ ,
\end{equation}
where $\delta_{ij}=1$ if $i =j$ and $\delta_{ij}=0$, otherwise. The solution of this problem establishes RQ as an eigenspace extractor with orthogonal nested minimizers.

\begin{theorem}(Proof in \Cref{app:proof-rq-eigenspace-extractor})\label{thm:rq-eigenspace-extractor}
    The RQ objective in \Cref{eq:rq-svd} is an eigenspace extractor for $T$, \ie $\sspan(\hat\phi_1,\dots,\hat\phi_d) = \sspan(\phi^*_1,\dots,\phi^*_d)$ and $\sspan(\hat\psi_1,\dots,\hat\psi_d) = \sspan(\psi^*_1,\dots,\psi^*_d)$. Furthermore, it admits orthogonal nested minimizers, $\hat \phi_i = c_i \phi^*_i$ and $\hat \psi_i = c_i \psi^*_i$ for $c_i = \pm 1$.
\end{theorem}
\Cref{thm:rq-eigenspace-extractor} implies that we can apply the methods in \Cref{sec:eigenspace-to-eigenfunctions} to extract exact eigenfunctions using the RQ objective, similarly to \Cref{thm:lora-eigenspace-extractor} for LoRA.

\paragraph{EVD Case.} 
We can specialize RQ to the EVD case, which is the setting considered by prior work~\citep{pfau2019spectral,deng2022neural}. Specifically, for a self-adjoint $T: \mcH \to \mcH$, we can optimize one set of functions: $\min_{\phi_i \in \mcH}  -\sum_{i=1}^{d} \braket{\phi_i}{T \phi_i}_{\mcH}$ such that $\braket{\phi_i}{\phi_j}_{\mcH} = \delta_{ij}$.
Additionally, we establish another objective that extracts eigenfunctions (up to permutation) directly. 

\begin{theorem}[Proof in \Cref{app:proof-rq-direct-eigenfunctions}]\label{thm:rq-direct-eigenfunctions}
    The following objective (\Cref{eq:rq-direct-eigenfunctions-appendix}) obtains a permutation of the true eigenfunctions directly, \ie the minimizers $(\hat\phi_i)_{i\in [d]}$ satisfy $\hat \phi_i = \pm \phi^*_{\sigma(i)}$ for some permutation $\sigma:[d]\mapsto[d]$.
    \begin{equation}\label{eq:rq-direct-eigenfunctions}
        \min_{\substack{\phi_i \in \mcH}} - \sum_{i=1}^{d} \braket{\phi_i}{ T  \phi_i}_{\mcH}\;\;\st\;\;
        \begin{aligned}
            \braket{\phi_i}{\phi_i}_{\mcH} &= 1, \; \forall i. \\
            \braket{\phi_i}{T \phi_j}_{\mcH} &= 0, \;\forall i \neq j. \\
        \end{aligned}
    \end{equation}
\end{theorem}
Compared to the previous Rayleigh quotient objective, \Cref{eq:rq-direct-eigenfunctions} imposes a modified orthogonality constraint $\braket{\phi_i}{T \phi_j}_{\mcH} = 0, \;\forall i \neq j$. This change alters the identifiability of the minimizer of the objective--from being unique only up to an orthonormal transformation to being identifiable up to a permutation of functions. Remarkably, this result generalizes the findings of \citet{deng2022neural}, who established this property for EVD only for $d=1$. Finally, we note that recovering the eigenvalues (so the precise ordering of the eigenfunctions) can be done by computing $\hat \lambda_i = \braket{\hat \phi_i}{T \hat \phi_i}$, for $i\in[d]$.

\paragraph{Contextual Kernel.} We can specialize the general RQ objective to the contextual kernel. By setting $T=\txa$, \Cref{eq:rq-svd} becomes
\begin{equation}\label{eq:rq-svd-population}
    \begin{aligned}
        \minimize_{\Phi, \Psi}
        \; &\E_{(x,a) \sim P^+(x,a)}\big[\norm{ \Phi(x) -  \Psi(a)}_2^2\big], \;\; \st \;\; 
        \braket{\phi_i}{\phi_j}_\px =  \braket{\psi_i}{\psi_j}_{\pa} = \delta_{ij} \ ,
    \end{aligned}
\end{equation}
where we added the constant $\E_{(x,a) \sim P^+}[\norm{\Phi(x)}_2^2 + \norm{\Psi(a)}_2^2] = 2d$ to the objective. This squared-distance form yields a more stable training process.

\paragraph{Connection to Non-contrastive Learning.} Interestingly, in the EVD case, the RQ objective recovers a well-known non-contrastive self-supervised learning objective, where \Cref{eq:rq-svd-population} becomes
\begin{equation}\label{eq:rq-population}
    \begin{aligned}
        \min_{\Psi} \; \E_{x \sim \px} & \E_{a, a^+ \sim P^+(a|x)}  [ \norm{ \Psi(a) -  \Psi(a^+)}_2^2] \;\; \st \;\; \braket{\psi_i}{\psi_j}_{\pa} = \delta_{ij} \ .
    \end{aligned}
\end{equation}
To solve this constrained problem, we can use the Lagrangian penalty functions (\Cref{sec:constrained-to-unconstrained}) as
\begin{equation}\label{eq:rq-loss}
    \begin{aligned}
        \mcL_{\rm RQ} &= \E_{x \sim \px} \E_{a, a^+ \sim P^+(a|x)}  \big[\norm{\Psi(a) -  \Psi(a^+)}_2^2\big] \\
        & + \frac{\mu}{d} \sum_{i=1}^d \big(\E_{a \sim P(a)} [\psi_i(a)^2] -1\big)^2  + \frac{\nu}{d(d-1)} \sum_{i\neq j} \big(\E_{a \sim P(a)} [\psi_i(a)\psi_j(a)]\big)^2 \ , 
    \end{aligned}    
\end{equation}
where $\mu, \nu>0$ are hyperparameters. Finally, we can use the sample splitting trick (\Cref{sec:constrained-to-unconstrained}) to estimate the penalty terms in an unbiased manner.

\paragraph{Comparison to VICReg~\citep{bardes2022vicreg}.}
We note that \Cref{eq:rq-loss} resembles the VICReg objective, as its three components correspond to invariance, variance, and covariance terms. However, VICReg employs a hinge loss penalty on the variance term, in contrast to the squared loss penalty used here. Also, the original implementation is biased in estimating the population covariance matrix, as it does not apply sample splitting. A detailed discussion of their empirical performance is provided in \Cref{app:compare_rq_vicreg}.

\section{Experiments}\label{sec:experiments}

We compare various approaches in our proposed framework on both synthetic and real-world datasets. Implementation details and additional results for both settings are provided in \Cref{app:exp-details}.

\subsection{Verification on Synthetic Kernels}\label{sec:exp-synthetic}

We start with synthetic kernels with known ground-truth decompositions, offering a direct signal on the performance of eigendecomposition. In this experiment, we compare the performance of LoRA and RQ objectives. To extract eigenfunctions, we use both joint nesting and Rayleigh-Ritz strategies, and exclude the sequential nesting option since its high computational cost makes it impractical for representation learning. 

\paragraph{Setup.} We construct synthetic kernels using orthonormal function bases.  Formally, to construct rank-$r$ kernel, we let $\kaa(a,a') = \sum_{i=1}^{r} \lambda_i \psi_i(a) \psi_i(a')$, where $\psi_i$ is either $(i-1)^\textrm{th}$ Legendre polynomial or $(i-1)^\textrm{th}$ Fourier base functions, with exponentially decayed eigenvalues $\lambda_i \propto \exp(-0.3i)$. We let $\pa(a) = 1/2^p$ be uniform over $[-1,1]^p$ and derive $P^+(a,a')$ from the definition of $\kaa$ accordingly. We set the input dimension to $p=1$ and $p=2$ for Legendre and Fourier kernels, respectively. Since we find that Fourier kernels are easier to learn, we adopt a more challenging setting for them compared to the Legendre kernels.

\paragraph{Evaluation Metrics.}
Let $\{(\lambda_i,\psi_i)\}_{i \in [r]}$ be the true eigenpairs and $\{(\hat\lambda_i, \hat\psi_i)\}_{i\in[d]}$ be the estimates up to rank~$d$. \looseness=-1

\begin{itemize}
     \item \textbf{Eigenvalue relative absolute error (EV-RAE)}: RAE is defined as
     $\frac{1}{d} \sum_{i=1}^{d} \frac{|\lambda_i - \hat{\lambda}_i|}{\lambda_i}$.
    
     \item \textbf{Eigenfunction mean squared error (EF-MSE)}: We measure the MSE defined on $\lap$:
     \begin{equation}
     \frac{1}{d}\sum_{i=1}^{d} \norm{\psi_i - \hat{\psi}_i}_\pa^2 
     = \frac{1}{d}\sum_{i=1}^{d} \E_{a \sim \pa} \left[ \left(\psi_i(a) - \hat{\psi}_i(a) \right)^2 \right].
     \end{equation}
     Since eigenfunctions can only be recovered up to a sign ambiguity, we select either $\hat{\psi}_i$ or $-\hat{\psi}_i$ that minimizes the MSE for each $i$.
 \end{itemize}

\begin{table}[htbp]
\centering
\resizebox{\textwidth}{!}{
\begin{tabular}{c|c|c|c|c|c|c|c}
\toprule
\multicolumn{2}{c|}{Group} & \multicolumn{2}{c|}{Low-rank} & \multicolumn{4}{c}{Rayleigh Quotient}  \\
\midrule
\multicolumn{2}{c|}{Objectives} & \multicolumn{2}{c|}{SCL (\Cref{eq:spectral-contrastive-loss}) } & 
\multicolumn{2}{c|}{RQ (\Cref{eq:rq-loss})} &
\multicolumn{2}{c}{VICReg (\Cref{eqn:vicreg})}  \\
\midrule
\multicolumn{2}{c|}{Eigenspace to eigenfunctions} & Rayleigh Ritz & Joint nesting & Rayleigh Ritz & Joint nesting & Rayleigh Ritz & Joint nesting  \\
\midrule \midrule
\multicolumn{8}{c}{Legendre Kernels} \\
\toprule
$p=1,r=6$ & EF(MSE) & $0.111_{ \pm 0.049 }$  & $\mathbf{0.068}_{ \pm 0.015 }$  & $0.085_{ \pm 0.005 }$  & $0.118_{ \pm 0.005 }$  & $0.254_{ \pm 0.002 }$  & $0.293_{ \pm 0.006 }$  \\ 
\cmidrule{2-8} 
 $d=3$ & EV(RAE) & $\mathbf{0.045}_{ \pm 0.004 }$  & $0.049_{ \pm 0.005 }$  & $0.071_{ \pm 0.002 }$  & $0.057_{ \pm 0.003 }$  & $0.067_{ \pm 0.002 }$  & $0.073_{ \pm 0.005 }$  \\ 
 \midrule
$p=1,r=8$ & EF(MSE) & $0.222_{ \pm 0.031 }$  & $\mathbf{0.077}_{ \pm 0.010 }$  & $0.122_{ \pm 0.003 }$  & $0.136_{ \pm 0.004 }$  & $0.397_{ \pm 0.007 }$  & $0.844_{ \pm 0.091 }$  \\ 
\cmidrule{2-8} 
 $d=4$ & EV(RAE) & $0.068_{ \pm 0.017 }$  & $0.071_{ \pm 0.015 }$  & $0.080_{ \pm 0.002 }$  & $\mathbf{0.067}_{ \pm 0.001 }$  & $0.131_{ \pm 0.007 }$  & $0.179_{ \pm 0.014 }$  \\ 
 \midrule
$p=1,r=10$ & EF(MSE) & $0.229_{ \pm 0.039 }$  & $\mathbf{0.109}_{ \pm 0.010 }$  & $0.153_{ \pm 0.005 }$  & $0.233_{ \pm 0.026 }$  & $0.480_{ \pm 0.003 }$  & $1.000_{ \pm 0.050 }$  \\ 
\cmidrule{2-8} 
$d=5$ & EV(RAE) & $0.075_{ \pm 0.009 }$  & $\mathbf{0.066}_{ \pm 0.010 }$  & $0.101_{ \pm 0.003 }$  & $0.092_{ \pm 0.004 }$  & $0.173_{ \pm 0.003 }$  & $0.300_{ \pm 0.064 }$  \\ 
\midrule \midrule
\multicolumn{8}{c}{Fourier Kernels} \\
\toprule
$p=2,r=6$ & EF(MSE) & $0.065_{ \pm 0.005 }$  & $\mathbf{0.058}_{ \pm 0.003 }$  & $0.088_{ \pm 0.003 }$  & $0.091_{ \pm 0.002 }$  & $0.200_{ \pm 0.001 }$  & $0.218_{ \pm 0.001 }$  \\ 
\cmidrule{2-8} 
 $d=3$ & EV(RAE) & $\mathbf{0.023}_{ \pm 0.006 }$  & $0.028_{ \pm 0.007 }$  & $0.062_{ \pm 0.004 }$  & $0.046_{ \pm 0.003 }$  & $0.043_{ \pm 0.004 }$  & $0.052_{ \pm 0.004 }$  \\ 
 \midrule
$p=2,r=8$ & EF(MSE) & $0.180_{ \pm 0.020 }$  & $\mathbf{0.105}_{ \pm 0.002 }$  & $0.133_{ \pm 0.003 }$  & $0.134_{ \pm 0.001 }$  & $0.448_{ \pm 0.153 }$  & $0.554_{ \pm 0.053 }$  \\ 
\cmidrule{2-8} 
 $d=4$ & EV(RAE) & $0.046_{ \pm 0.011 }$  & $\mathbf{0.044}_{ \pm 0.011 }$  & $0.071_{ \pm 0.002 }$  & $0.064_{ \pm 0.003 }$  & $0.183_{ \pm 0.083 }$  & $0.238_{ \pm 0.036 }$  \\ 
 \midrule
$p=2,r=10$  & EF(MSE) & $0.195_{ \pm 0.016 }$  & $\mathbf{0.133}_{ \pm 0.003 }$  & $0.179_{ \pm 0.005 }$  & $0.183_{ \pm 0.003 }$  & $0.454_{ \pm 0.039 }$  & $0.812_{ \pm 0.073 }$  \\ 
\cmidrule{2-8} 
$d=5$ & EV(RAE) & $\mathbf{0.054}_{ \pm 0.010 }$  & $0.054_{ \pm 0.009 }$  & $0.077_{ \pm 0.003 }$  & $0.069_{ \pm 0.003 }$  & $0.218_{ \pm 0.035 }$  & $0.370_{ \pm 0.038 }$  \\ 
 \midrule
\bottomrule
\end{tabular}}
\caption{Comparison of different methods for estimating eigenfunctions and eigenvalues on the Legendre and Fourier kernels. Each value is $\text{mean}_{\text{std.error}}$ over $5$ runs.
}
\label{tab:synthetic_comparison}
\end{table}

\begin{figure}[t!]
    \centering
    \includegraphics[width=\textwidth]{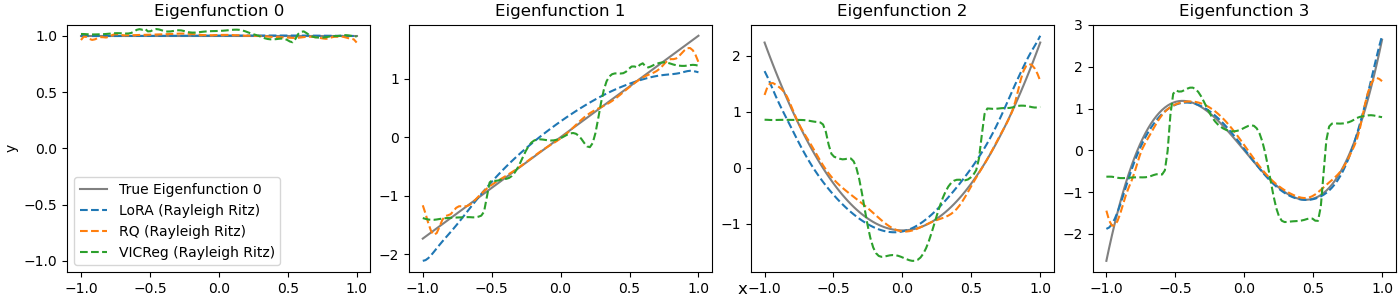}
    \caption{Comparison of eigenfunctions estimated by different methods (LoRA, RQ, and VICReg) with the ground-truth eigenfunctions (solid gray line). Results are shown for Legendre kernels with Rayleigh–Ritz post-processing ($p=1$, $r=8$, $d=4$).}
    \label{fig:eigenfunction_comparison}
\end{figure}

\paragraph{Results.} Table~\ref{tab:synthetic_comparison} reports the average EF-MSE and EV-RAE across different objectives and optimization strategies. The results demonstrate several key findings. First, both LoRA and RQ consistently recover the top-$d$ eigenfunctions with low estimation errors across all configurations (varying $r$ and $d$), confirming our theoretical results. Among them, LoRA with joint nesting achieves the most accurate extraction, corroborating the empirical results in \citep{ryu2024operator}.

Moreover, both joint nesting and Rayleigh–Ritz effectively extracted eigenfunctions from the eigenspace, consistent with our analysis. In contrast, VICReg consistently produces poor estimates. This highlights the drawback of penalizing empirical covariance rather than population covariance, which leads to biased estimation.
These observations are visually confirmed in \Cref{fig:eigenfunction_comparison}, where the reconstructed eigenfunctions of all methods closely match the true eigenfunctions, except for VICReg. Additional visualizations are provided in \Cref{app:synthetic-exp}.

\subsection{Feature Importance for Adaptive-dimensional Representations}\label{sec:real-world-exp}

In this experiment, we demonstrate that ordered eigenfunctions enable adaptive-dimensional representations. Treating eigenvalues as feature-importance scores and ranking eigenfunctions accordingly, the representation can be truncated to any desired dimension. This provides a principled way to balance accuracy and efficiency: higher-dimensional features generally improve accuracy but demand greater computational resources, while lower-dimensional features reduce storage and computation at the cost of performance. 

\paragraph{Setup.} We follow the standard self-supervised learning protocol on CIFAR-10~\citep{krizhevsky2009learning} and ImageNette, which consists the first 10 classes of ImageNet~\citep{deng2009imagenet}. We pretrain the ResNet-18 with VICReg~\citep{bardes2022vicreg} and SCL objectives for eigenspace extraction.
We set the dimension of the full representation $d=512$ and evaluate the performance of the adaptive dimensions of lengths $r=\{4,8,16,32,64,128,256\}$. We use the standard linear probe accuracy for evaluation. For SCL, we follow the original implementation~\citep{haochen2021provable}, applying $\ell_2$ normalization to the final representation before computing the loss (\Cref{eq:spectral-contrastive-loss}), \ie $\Psi'(a) = \Psi(a) / \norm{\Psi(a)}_2$. This normalization substantially stabilizes training and improves performance (by about 10\% on ImageNette).

We compare the two eigenfunction extraction approaches, joint nesting (JN) and Rayleigh-Ritz (RR), with independently trained low-dimensional fixed feature (FF) representations, and randomly selected (RS) features on the full representations (whose results are averaged over $300$ runs). We note that the joint nesting model is separately trained under the same settings, and Rayleigh-Ritz method is applied on the single full-representation model. 

\begin{figure}[t]
  \centering
  \begin{subfigure}[t]{0.24\linewidth}
      \centering
      \includegraphics[width=\linewidth]{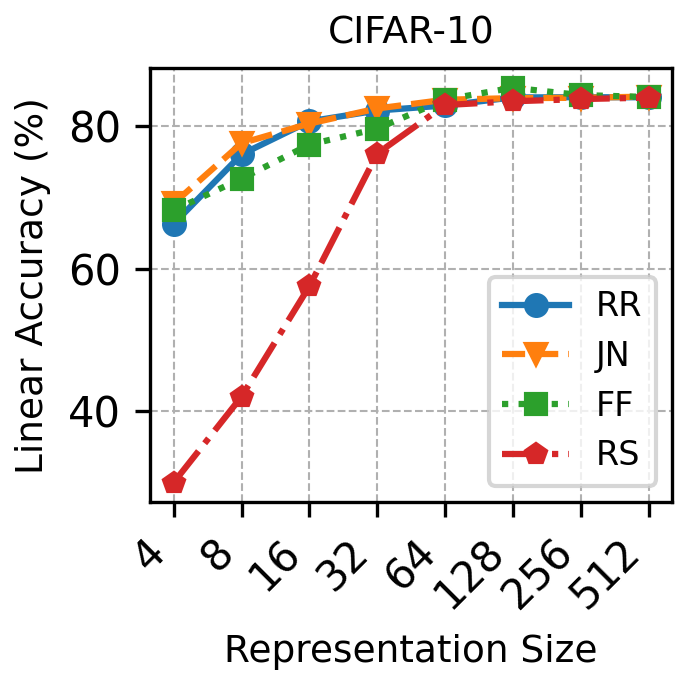}
      \caption{VICreg on CIFAR-10.}
      \label{fig:cifar10_vicfreg}
  \end{subfigure}
  \hfill
  \begin{subfigure}[t]{0.24\linewidth}
      \centering
      \includegraphics[width=\linewidth]{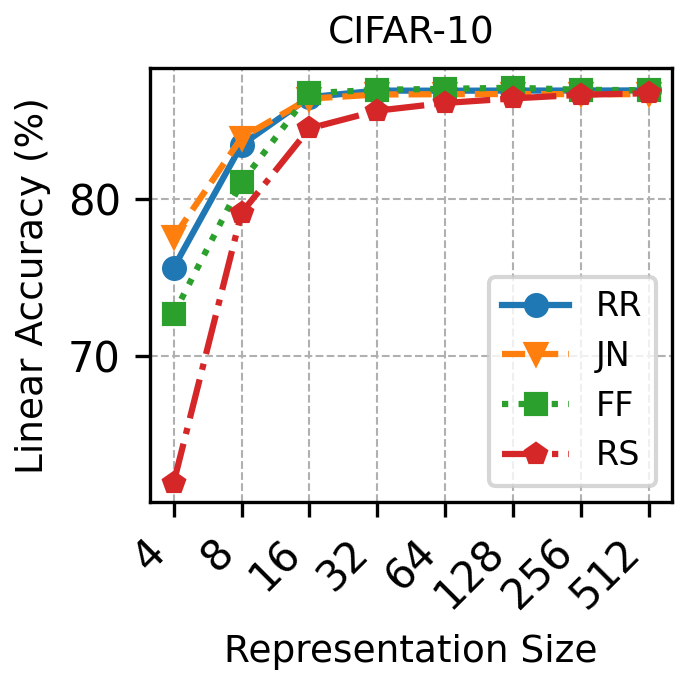}
      \caption{SCL on CIFAR-10.}
      \label{fig:cifar10_scl}
  \end{subfigure}
  \hfill
  \begin{subfigure}[t]{0.24\linewidth}
      \centering
      \includegraphics[width=\linewidth]{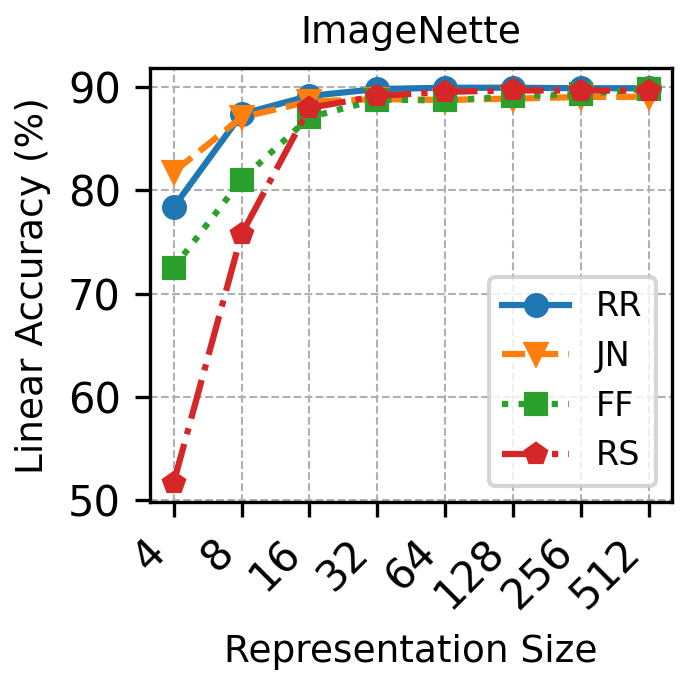}
      \caption{VICReg on ImageNette.}
      \label{fig:imagenette_vicreg}
  \end{subfigure}
  \hfill
  \begin{subfigure}[t]{0.24\linewidth}
      \centering
      \includegraphics[width=\linewidth]{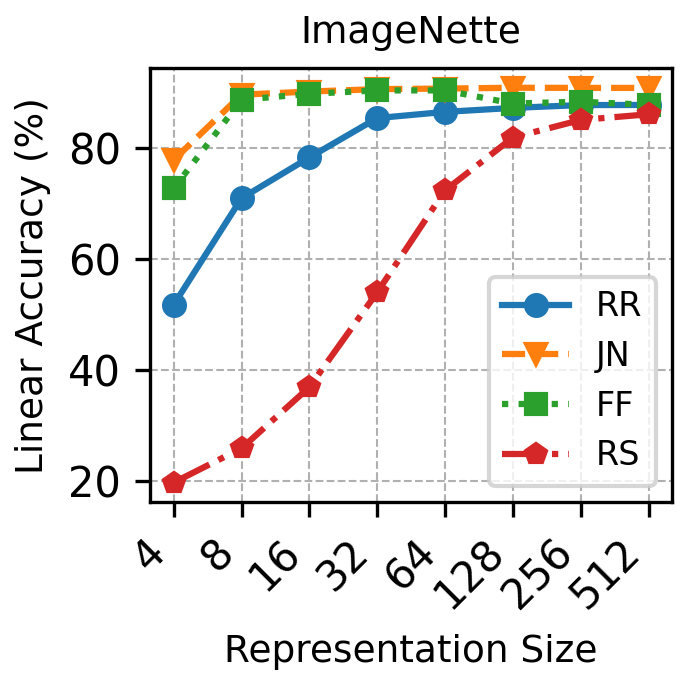}
      \caption{SCL on ImageNette.}
      \label{fig:imagenette_scl}
  \end{subfigure}
  \caption{Linear evaluations of VICReg and SCL on different representation sizes on CIFAR-10 and ImageNette.}
  \label{fig:image_results}
\end{figure}

\paragraph{Results.} \Cref{fig:image_results} reports the linear probe accuracies of VICReg and SCL on ImageNette and CIFAR-10. As expected, the accuracy increases with the dimension of the representation. In the rightmost three panels, joint nesting \textbf{(JN)} and Rayleigh–Ritz \textbf{(RR)} consistently match or surpass fixed features \textbf{(FF)} across all dimensions, while all three methods substantially outperform random selection \textbf{(RS)} at lower dimensions. These results highlight the strong efficiency–accuracy trade-off achieved by JN and RR, especially in low-dimensional regimes.\\

For SCL on ImageNette, RR notably underperforms both FF and JN. We hypothesize that this is due to the $\ell_2$-normalization of SCL's output representations. This normalization causes a much slower eigenvalue decay compared to VICReg, as illustrated in \Cref{fig:decay}, which means more eigenfunctions are of similar importance. Consequently, truncating the representation, as RR does, results in significant information loss.

\begin{wrapfigure}{r}{0.3\textwidth}
    \centering
    \includegraphics[width=\linewidth]{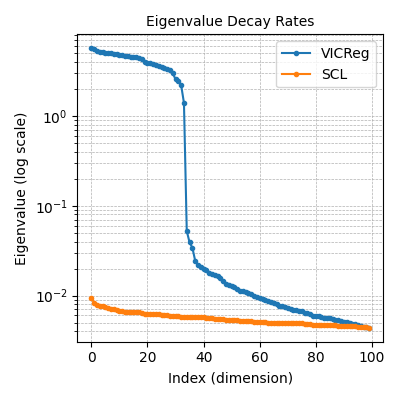}
    \caption{Top 100 eigenvalues of SCL and VICReg on ImageNette. }
    \label{fig:decay}
\end{wrapfigure}

In contrast, JN avoids this issue by operating on the representations before normalization, \ie $\Psi'_r(a) = \Psi_r(a) / \norm{\Psi_r(a)}_2$ for all $r$, effectively learning features on hyperspheres of different sizes, and achieves consistently better performance than FF. From a theoretical standpoint, the $\ell_2$-normalization in SCL prevents the objective's minimizers from corresponding to the true eigenspace, which is why our theoretical analysis does not directly apply.

From a practical standpoint, we recommend using RR for VICReg, as it is more flexible than JN: reduced-dimensional representations can be obtained from a single full model through post-processing, whereas JN requires specifying the target dimension during training. This makes RR particularly appealing for proprietary embedding models, where eigenpairs can be estimated without access to model weights or training pipelines. 
For SCL, as the commonly used $\ell_2$-normalization stabilization trick can degrade RR’s performance, we recommend using JN instead.

\section{Conclusion and Future Work} \label{sec:conclusion}

We present a unified framework for extracting ordered and identifiable eigenfunctions to advance representation learning. We formalize key desiderata, provide modular methods, and connect contrastive and non-contrastive methods to the two major eigendecomposition paradigms. Experiments on image datasets validate that ordered eigenfunctions enable adaptive-dimensional representations. This framework provides a principled foundation for building ordered representation learning systems. 

We conclude with two promising future directions. First, recent studies \citep{reizinger2025position,pmlr-v235-huh24a} have shown that representations learned from contrastive and non-contrastive methods are remarkably similar. We hypothesize that this similarity can be explained by our identifiability results in \Cref{sec:variational-results}, which suggest that both approaches effectively extract the eigenspace of the contextual kernel. A natural next step is to conduct larger-scale experiments to further test this hypothesis and extend our analysis beyond the current LoRA and RQ objectives. The second direction is to explore alternative norms beyond the Hilbert–Schmidt norm. More generally, one may ask how to obtain a top-$d$ approximation of a given kernel under arbitrary distance measures. Investigating broader classes of norms and objectives that are both theoretically grounded and computationally tractable represents an exciting avenue for future work.

\subsubsection*{Acknowledgements}
We thank Runtian Zhai and Randall Balestriero for helpful discussions and feedback on this paper. We acknowledge the support of AFRL and DARPA via FA8750-23-2-1015, ONR via
N00014-23-1-2368, and NSF via IIS-1909816.

\clearpage

\appendix

\doparttoc 
\faketableofcontents 

\part{Appendix} 
\parttoc 

\clearpage

\section{Additional Discussion on Related Work}\label{app:related-work}

\subsection{Adaptive Representation Learning} \label{app:related_work_adaptive_representations}
Matryoshka Representation Learning (MRL)~\citep{kusupati2022matryoshka} first introduced a training framework for learning flexible representations that can adapt to tasks with varying computational budgets. The core idea aligns with the principle of joint nesting, where objectives are summed across embedding lengths derived from the same full-dimensional vector. Since its introduction, this concept has been extended to a wide range of applications, including large language models~\citep{openai2024embedding,nussbaum2025nomic,yu2025arcticembed,vera2025embeddinggemma}, diffusion models~\citep{gu2023matryoshka}, multi-modal models~\citep{cai2025matryoshka,hu2024matryoshka}, recommender systems~\citep{wang-etal-2024-train}, text embeddings~\citep{cai2025matryoshka}, sparse autoencoder~\citep{bussmann2025learning}, and federated learning~\citep {NEURIPS2024_7a9ee756}.

Structured representations can also be achieved through various methods. \citet{wang2024nonnegative} introduced non-negative contrastive learning, showing that sparse yet expressive representations can be learned with deep encoders. In this approach, representations remain high-dimensional but can be stored in sparse matrices, substantially reducing computational and storage costs. This idea has also been demonstrated in state-of-the-art embedding models~\citep{wen2025beyond}. Other approaches incorporate additional structure into the learned representations, such as hierarchical representations~\citep{desai2023hyperbolic,you2025hierarchical} and principal component analysis (PCA)–based representations~\citep{zhu2018self}. Furthermore, \citet{xu2024neural} used nesting to learn structured representations that decompose multivariate dependencies without the need for retraining.

We note that our eigenfunction extraction framework provides a theoretical foundation for understanding the success of structured representation methods, particularly MRL and PCA-based representations. By combining MRL with the spectral perspective of representation learning, we demonstrate that MRL effectively extracts the ordered eigenfunctions of the representation space. Consequently, features associated with larger eigenvalues naturally emerge in the earlier dimensions. Furthermore, our framework reveals that eigenfunctions can be extracted through a PCA-style procedure, the Rayleigh–Ritz method. This leads to an important conclusion: essential features can be recovered directly from a single full-representation model, potentially alleviating the need for MRL.

\subsection{Theories of Representation Learning}
The theoretical study of representations has been a central theme in machine learning research. The most relevant line of work for this paper is from spectral theory. From a spectral perspective, \citet{haochen2021provable} first introduced the spectral contrastive loss, derived by spectrally decomposing the positive-pair kernel using a low-rank approximation. \citet{johnson2023contrastive} further showed that several widely used objectives, such as NT-Logistic and NT-XEnt, implicitly approximate the same kernel. More recently, \citet{zhai2025contextures} demonstrated that the VICReg objective can be interpreted as performing a Rayleigh quotient–based kernel decomposition. Building on this perspective, follow-up studies proposed various practical extensions, including feature collapse detection~\citep{garrido2023rankme}, representation alignments~\citep{fumero2024latent}, and zero-mean regularization~\citep{zhou2024zero}.

In contrast, our work unifies the insights from \citet{haochen2021provable,johnson2023contrastive,zhai2025contextures} and introduces a general framework for extracting eigenfunctions through their eigenspace extractors. This framework provides a principled way to order learned representations according to their importance scores, thereby enabling the construction of adaptive representations. In a closely related study, \citet{zhang2024identifiable} propose a tri-factor contrastive learning objective that extracts eigenvalues by enforcing orthonormality constraints within the contrastive objective. However, their method is restricted to contrastive learning settings, making it less general than our eigenfunction extraction framework.

Other works approach representation learning from diverse theoretical perspectives. \cite{garrido2023on} shows that contrastive objectives and non-contrastive methods are equivalent under certain assumptions.
\cite{hjelm2018learning, Tschannen2020On} show that representation learning can be seen as maximizing mutual information between inputs and their latent representations. 
\cite{zhai2024understanding,wang2024understanding,pmlr-v162-saunshi22a,wang2022chaos,lee2021predicting, brady2023provably,brady2025interaction} study how learned representations transfer across tasks and distributions. The works on feature dynamics~\citep{pmlr-v119-wang20k,wang2023message,huang2025a,wen2021toward} and equivariance/invariance~\citep{gupta2024structuring,pmlr-v202-garrido23b,xiao2021what} focus on how representations evolve during training and how structural constraints, such as symmetry or transformation consistency, shape the learned embedding space. 

\cite{reizinger2024cross,daunhawer2023identifiability,pmlr-v258-rusak25a,lyu2022understanding} propose identifiability theories showing that these representation learning techniques recover some particular data generation process. 
An alternative modeling of the representation space is proposed by the causal representation learning literature, which aims to form a causal understanding of the world by learning appropriate representations that support causal interventions, reasoning, and planning~\citep{scholkopf2021toward}.
Specifically, this approach proposes building interpretable and identifiable models from the ground up, and the identifiability of \emph{causal} representations has been extensively studied in various setups~\citep{ahuja2023interventional,varici2025score,yao2025unifying,liang2023causal,moran2025towards}. Inspired by this progress, follow-up work studied extracting identifiable representations from data with a geometric notion of \emph{concept}~\citep{rajendran2024from}, rooted in the linear representation hypothesis~\citep{jiang2024origins}.

\subsection{Algorithms of Representation Learning }
Modern representation learning approaches train deep encoders to produce representations that generalize to downstream tasks. Early methods relied on generative or context-based pretext tasks~\citep{gidaris2018unsupervised}, but recent advances center on instance discrimination via paired augmentations. \textbf{Contrastive methods} generate two views from the same input and train the network to maximize agreement of positive pairs (same inputs) while pushing apart negatives (different inputs). Prominent examples include SimCLR
~\citep{chen2020simple, oord2018representation}, MoCo~\citep{he2020momentum}, and SCL~\citep{haochen2021provable}. \textbf{Non-contrastive methods} remove explicit negatives and rely on architectural biases or regularizers to avoid collapse. Prominent examples include BYOL~\citep{NEURIPS2020_f3ada80d}, SimSIAM~\citep{chen2021exploring}, VICReg~\citep{bardes2022vicreg}, Barlow Twins~\citep{zbontar2021barlow}, SWAV~\citep{caron2020unsupervised}, DINO~\citep{caron2021emerging,oquab2024dinov,simeoni2025dinov3}, and JEPA~\citep{assran2023self, assran2025v}. These methods can also be applied to multi-modal learning, where shared representation spaces are learned from information from aligned data from different modalities~\citep{CLIP, jia2021scaling, yuan2021florence, li2022blip,chuang2025meta,tschannen2025siglip}. Other representation learning methods include masked language/image modeling~\citep{devlin2018bert,he2022masked,chang2023muse}. See the self-supervised-learning cookbook~\citep{Balestriero2023ACO} for a summary of these methods.

\section{Eigendecomposition of Finite-Dimensional Matrices}\label{app:matrix-EVD}

We provide some background on the eigendecomposition of finite-dimensional matrices as the precursor for the optimization viewpoint for two main approaches for kernel decomposition: low-rank approximation and Rayleigh quotient optimization.

Consider a symmetric positive semidefinite matrix $\bM \in \R^{N \times N}$ of rank $n$. Formally, the rank of a matrix is defined as the dimension of its column (or row) space, \ie the maximum number of linearly independent columns or rows. Thus, a rank-$n$ symmetric matrix can always be expressed as the sum of exactly $n$ rank-one matrices:
\begin{equation}
    \bM = \sum_{i=1}^{n} \bu_i \bu_i^\top \ , \quad \bu_i \in \R^N \ .
\end{equation}
However, this decomposition is not unique without further constraints, as there exist infinitely many ways to write $\bM$ as a sum of rank-$1$ matrices. Uniqueness arises precisely when we require the vectors forming the rank-$1$ matrices to be orthonormal, \ie 
\begin{equation}
    \bu_i = \sqrt{\lambda_i} \bv_i \ , \quad \forall i \in [n] \ , \quad \text{where} \quad \bv_i^\top \bv_j = \delta_{ij} \ ,
\end{equation}
where $\delta_{ij}=1$ for $i=j$ and $0$ otherwise. Here, $(\lambda_i)_{i \in [n]}$ are the positive eigenvalues of $\bM$  and $(\bv_i)_{i \in [n]}$ are the corresponding eigenvectors. Assuming $\lambda_1 > \dots > \lambda_n > 0$, we have a \emph{unique} eigendecomposition of~$\bM$,
\begin{equation}
    \bM = \sum_{i=1}^{n} \lambda_i \bv_i \bv_i^\top \ , \quad \bv_i \in \R^N \ .
\end{equation}

Before moving on to specific methods, we emphasize what makes learning the top eigenvectors appealing: given a fixed number $d<n$, the eigenvectors corresponding to the top-$d$ eigenvalues can be used to closely approximate the original matrix $\bM$. Formally, let $\bM_d$ denote the top-$d$ truncated eigendecomposition of $\bM$ as
\begin{equation}\label{eq:def-truncated-evd}
    \bM_d \coloneq \sum_{i=1}^{d} \lambda_i \bv_i \bv_i^\top \ .
\end{equation}
\Cref{thm:SEYM}(ii) states that the truncated eigendecomposition provides the best low-rank approximation of $\bM$. Specifically, for symmetric positive semidefinite matrices $\bM, \bA \in \R^{N \times N}$ where $\rank(\bA) = d  \leq \rank (\bM) = n$, we have
\begin{equation}\label{eq:EYM-matrices}
    \norm{\bM - \bM_d} \leq \norm{\bM - \bA}  \ . 
\end{equation}

\paragraph{Optimization viewpoint.} 
Given a symmetric positive definite matrix $\bM \in \R^{n \times n}$, denote its eigendecomposition by $\bM = \bV \Sigma \bV^\top$, where $\bV \in \R^{n \times n}$ has columns that are the eigenvectors, and $\bSigma = \diag(\lambda_1,\cdots,\lambda_n)$ with $\lambda_1 \geq \dots \geq \lambda_n > 0$. For clarity of the exposition, assume that $\lambda_d$ has multiplicity 1, \ie $\lambda_d > \lambda_{d+1}$. The main objective is to recover the top-$d$ eigenvectors, $\bv_1,\cdots,\bv_d \in \R^{n}$. As a simpler precursor to this objective, one may consider recovering the top-$d$ \textbf{eigenspace}, that is, the column space of the matrix $\bV_d \coloneq [\bv_1,\cdots,\bv_d] \in \R^{n \times d}$. To recover the top-$d$ eigenspace, we consider two classical approaches that can be formulated as optimization problems: low-rank matrix approximation and Rayleigh quotient optimization.
\begin{itemize}[leftmargin=*]
    \item \textbf{Low-rank matrix approximation} can be formulated as an unconstrained optimization problem. Specifically, we can minimize the Frobenius norm of the approximation error:
    \begin{equation}\label{eq:lora-matrix-generic}
        \min_{\bU \in \R^{n \times d}} \norm{\bM - \bU \bU^\top|}_{\fro}^2 \ .
    \end{equation}
    \Cref{thm:SEYM}(ii) readily implies that the objective is minimized if and only if $\bU \bU^\top = \bM_d$. Note that by definition, $\bM_d = \bV_d \bSigma_d \bV_d^\top$, where $\bSigma_d = \diag(\lambda_1,\dots,\lambda_d)$, and $\bU = \bV_d \sqrt{\bSigma_d}$ is a solution. However, the solution is not unique as $\hat\bU = \bU \bQ = \bV_d \sqrt{\bSigma_d} \bQ $ for any orthonormal matrix $\bQ \in \R^{d \times d}$ also satisfies the equality. For a minimizer $\hat \bU$, the objective value becomes $\norm{\bM - \bV_d\bV_d^\top|}_{\fro}^2 = \sum_{i=d+1}^{n} \lambda_i^2$.

    \item \textbf{Rayleigh quotient optimization} aims to find the dominant eigenvectors of a matrix by maximizing the Rayleigh quotient:
    \begin{equation}
        R(\bM, \bu) \coloneq \frac{\bu^* \bM \bu}{\bu^* \bu} \ ,
    \end{equation}
    where $\bu^*$ denotes the conjugate transpose of $\bu$. The maximum value of $R(\bM, \bu)$ is equal to the largest eigenvalue $\lambda_1$ of $\bM$; for Hermitian matrices, this corresponds to the spectral norm of $\bM$. The eigenvector achieving this maximum is the unit-norm vector $\bv_1$ that satisfies $\bM \bv_1 = \lambda_1 \bv_1$. Subsequent top eigenpairs $(\lambda_i,\bv_i)$ for $i \in \{2,\dots,d\}$ can be obtained by sequentially maximizing the Rayleigh quotient over the orthogonal complement of the preceding eigenvectors. Similarly to the low-rank approximation objective above, if we relax the objective to recover only the top-$d$ eigenspace, then Rayleigh quotient optimization becomes equivalent to maximizing the following objective:
    \begin{equation}\label{eq:rayleigh-matrix-generic-eigenspace}
        \begin{aligned}
        \max_{\bU \in \R^{n \times d}}  & \trace(\bU^\top \bM \bU) 
        = \max_{\bU \in \R^{n \times d}}  \sum_{i=1}^{d} \bu_i^{\top}\bM\bu_i \ , \;\;
        \st & \bU^\top\bU = \mathbf{I}_d \ , 
        \end{aligned}
    \end{equation}
    where $\trace(\cdot)$ denotes the trace operation. The maximizer of the objective is given by $\hat \bU = \bV_d\bQ$, where $\bQ \in \R^{d \times d}$ is an orthonormal matrix. Note that the maximum value of the objective is $\sum_{i=1}^{d} \lambda_i$, the sum of the top-$d$ eigenvalues.
\end{itemize}
Therefore, solving either \Cref{eq:lora-matrix-generic} or \Cref{eq:rayleigh-matrix-generic-eigenspace} yields the top-$d$ eigenspace; that is, the column space of solution $\hat\bU$ satisfies 
\begin{equation}
    \sspan\{\hat\bu_1,\hat\bu_2,\cdots, \hat\bu_d \} = \sspan\{\bv_1, \bv_2, \cdots,\bv_d\} \ .
\end{equation}

\section{Proofs for the Eigenfunction Extraction Framework}\label{app:proofs-framework}

\subsection{\texorpdfstring{Proof of~\Cref{thm:sequential-nesting-general}}{Proof of Theorem~2} }\label{app:proof-sequential-nesting}

\paragraph{\Cref{thm:sequential-nesting-general}.}
{\it
    Assume that the base problems $\mcO_1,\dots,\mcO_d$ (i) are eigenspace extractors for a compact operator $T$, and (ii) admit orthogonal nested minimizers. Then, sequentially solving 
    \begin{equation}
        \tilde\mcO_j : \min_{\phi_j } \mcL_j\big( \hat \phi_1, \ldots, \hat \phi_{j-1}, \phi_j \big) \ , 
    \end{equation} 
    such that $(\hat \phi_1, \ldots, \hat \phi_{j-1}, \phi_j) \in \mcC_j$ and where $\hat \phi_i$ is the solution from step $i < j$, recovers the $j$-th eigenfunction of $T$ up to a scaling factor $c_j$, \ie $\hat \phi_j = c_j \phi^*_j$. }

\begin{proof}
     We prove the desired result by induction on $d$ to show that given the conditions in the theorem statement for $\mcO_1,\dots,\mcO_d$, we will obtain $\hat \phi_i = c_i \phi^*_i$ where $c_i \neq 0 \in \R$, $~\forall i \in [d]$,
    First, the base case $i=1$ follows from the fact that $\mcO_1$ is an eigenspace extractor, which implies $\hat \phi_1 = c_1 \phi^*_1$  for some nonzero scalar $c_1$.
    
    Now assume that the induction hypothesis is true up to $j-1$, that is, $\hat \phi_i = c_i \phi^*_i$ and $(\hat \phi_1, \ldots, \hat \phi_i) \in \argmin_{\phi_1,\dots,\phi_i} \mcO_{i}$ for all $i \in [j-1]$. Since $\mcO_1, \cdots, \mcO_j$ admit orthogonal nested minimizers, we know that given $(\hat \phi_1, \ldots, \hat \phi_{j-1}) \in \argmin \mcO_{j-1}$, the solution $\hat \phi_j \coloneq \argmin_{\phi_j \in \mcF} \mcO_j(\hat \phi_1, \ldots, \hat \phi_{j-1}, \phi_j)$ exists, and 
    \begin{equation}
        \min_{\phi_j \in \mcF} \tilde{\mcO}_j(\hat \phi_1, \ldots, \hat \phi_{j-1}, \phi_j) = \min_{(\phi_1,\cdots,\phi_d ) \in \mcC_j} \mcO_j( \phi_1, \ldots,  \phi_{j-1}, \phi_j).
    \end{equation}
    This implies that solving the sequential objective $\tilde{\mcO}_j$ achieves the same minimum as the original base problem $\mcO_j$, and $(\hat \phi_1, \ldots, \hat \phi_{j-1}, \hat \phi_j)$ is a minimizer of $\mcO_j$. Moreover, since $\mcO_j$ is an eigenspace extractor, we have  
    \begin{align}\label{eq:eigenspace-span-aux1}
        \sspan(\hat \phi_1, \ldots, \hat \phi_j) = \sspan(c_1 \phi^*_1, \ldots, c_{j-1} \phi^*_{j-1}, \hat \phi_j)  = \sspan(\phi^*_1, \ldots, \phi^*_j). 
    \end{align}
    Finally, since the orthogonal nested minimizer property says that $\hat \phi_j$ is orthogonal to all previous functions, \ie $\braket{\hat\phi_j}{\hat\phi_i}_{\mcH} = \braket{\hat\phi_j}{c_i\phi^*_i}_{\mcH}= 0$ for all $1 \leq i < j$, we must have $\hat \phi_j = c_j \phi^*_j$ for some nonzero $c_j \in \R$. This completes the induction step, and extending it to $j=d$ completes the proof. 
\end{proof}

\subsection{\texorpdfstring{Proof of~\Cref{thm:joint-nesting-general}}{Proof of Theorem~3}}\label{app:proof-joint-nesting}

\paragraph{\Cref{thm:joint-nesting-general}.}
{\it
    If $\mcO_1,\dots,\mcO_d$ are eigenspace extractors for $T$ with orthogonal nested minimizers, then for any positive weights $(w_i)_{i\in[d]}$, solving $\mcO_{\jnt}$ recovers the eigenfunctions up to scaling, \ie its minimizers are of the form $(c_1 \phi^*_1,\ldots, c_d\phi^*_d)$, where $(\phi^*_i)_{i\in[d]}$ are the top-$d$ eigenfunctions of $T$. }

\begin{proof} 
    Let $(\hat \phi_1,\dots,\hat \phi_d)$ denote an orthogonal nested minimizer of $\mcO_1,\dots,\mcO_d$. By definition, this implies
    \begin{equation}\label{eq:orthogonal-nested-joint-nesting-proof}
        (\hat \phi_1,\dots,\hat \phi_j) \in \argmin_{\phi_1,\dots,\phi_j} \mcO_j \ , \quad \forall j \in [d] \ .
    \end{equation}
    Also, let $(\phi'_1,\dots,\phi'_d)$ denote a minimizer of $\mcO_{\jnt}$ with weights $(w_1,\dots,w_d)$. \Cref{eq:orthogonal-nested-joint-nesting-proof} gives that
    \begin{equation}
        \mcL_j(\hat \phi_1,\dots,\hat \phi_j) \leq \mcL_j(\phi'_1,\dots,\phi'_j) \ , \quad \forall j \in [d] \ ,
    \end{equation}
    which immediately implies
    \begin{equation}
        \sum_{j=1}^{d} w_j \mcL_j(\hat \phi_1,\dots,\hat \phi_j) \leq \sum_{j=1}^{d} w_j\mcL_j(\phi'_1,\dots,\phi'_j) 
    \end{equation}
    for any nonnegative weights $(w_j)_{j \in [d]}$. Since $(\phi'_1,\dots,\phi'_d)$ is a minimizer of $\mcO_{\jnt}$, the inequality must hold with equality for all $j\in[d]$, \ie $\mcL_j(\hat \phi_1,\dots,\hat \phi_j) = \mcL_j(\phi'_1,\dots,\phi'_j)$.     Therefore, we also have 
    \begin{equation}\label{eq:nested-condition-joint-nesting-proof}
        (\phi'_1,\dots,\phi'_j) \in \argmin_{\phi_1,\dots,\phi_j} \mcO_j(\phi_1,\dots,\phi_j) \ .
    \end{equation}
    For $j=1$ this readily gives $\phi'_1 = \alpha_1 \phi^*_1$ for a nonzero scalar $c_1$. 
    Now recall that $\mcO_1,\dots,\mcO_d$ have the orthogonal nested minimizer property and observe that $(\phi'_1,\dots,\phi'_d)$ satisfy the first condition due to \Cref{eq:nested-condition-joint-nesting-proof}. Consequently, the second condition in \Cref{def:nested-minimizer} implies that $\phi'_j$ is orthogonal to all previous functions, \ie $\braket{\phi'_j}{\phi'_i}=0$ for $i < j$, and $(\phi'_1,\dots,\phi'_d)$ is also an orthogonal nested minimizer of $\mcO_1,\dots,\mcO_d$. Then, using the eigenspace extractor property, \Cref{thm:sequential-nesting-general} gives that $(\phi'_1,\dots,\phi'_d)$ is of the form $(c_1 \phi^*_1,\dots,c_d \phi^*_d)$ where $(\phi^*_i)_{i \in [d]}$ are the top-$d$ eigenfunctions of $T$.
\end{proof}

\subsection{\texorpdfstring{Proof of~\Cref{thm:rayleigh-ritz}}{Proof of Theorem 4}}\label{app:proof-rayleigh-ritz}

\paragraph{\Cref{thm:rayleigh-ritz}.}
{\it
    Suppose that $(\phi_i)_{i \in [d]}$ form an orthonormal basis for the top-$d$ eigenspace of a self-adjoint compact operator $T$. Let $\bB \in \R^{d \times d}$ be a matrix with entries $\bB_{ij} \vcentcolon = \braket{\phi_i}{T \phi_j}$. Denote the eigenpairs of $\bB$ by $\{(\lambda_i,\by_i) : i \in [d]\}$, where $(\lambda_i)_{i \in [d]}$ are in non-increasing order and $\{\by_i \in \R^d : i \in [d]\}$ are the corresponding eigenvectors. Then, the top $i^\textrm{th}$ eigenfunction of $T$ is given by $\phi^*_i = \sum_{j=1}^d (\by_i)_{j} \phi_j$, with the corresponding eigenvalue $\lambda^*_i=\lambda_i$.
}
\begin{proof}
    Substituting $\bB_{ij}= \braket{\phi_i}{T \phi_j}$ into $\bB y_{\ell} = \lambda_{\ell} \by_{\ell}$ for each $i, \ell \in [d]$ we have 
    \begin{align}
        \lambda_{\ell} {y}_{\ell_i} &= \sum_{j=1}^d \braket{\phi_i}{T \phi_j}y_{\ell_j} = \braket{\phi_i}{ T \big( \sum_{j=1}^d y_{\ell_j}  \phi_j \big)} \\
        \implies \lambda_{\ell} \sum_{i=1}^{d} {y}_{\ell_i} \ket{\phi_i} &= \sum_{i=1}^{d} \ketbra{\phi_i}{\phi_i}{T \big ( \sum_{j=1}^d y_{\ell_j}  \ket{\phi_j} \big ) } \\
        &= \Big( \sum_{i=1}^d \ketbra{\phi_i}{\phi_i} \Big ) {T\big ( \sum_{i=1}^d {y}_{\ell_i} \ket{\phi_i} \big ) } \ . \label{eq:rayleigh-ritz-aux1}
    \end{align}
    Since $T$ is compact and self-adjoint, by the spectral theorem $T = \sum_{j=1}^\infty \lambda^*_j \ketbra{\phi^*_j}{\phi^*_j}$. Next, since $\phi_1,\ldots,\phi_d$ is an orthonormal basis of the subspace spanned by the orthonormal eigenfunctions $\phi^*_1,\ldots,\phi^*_d$, for each $j\in[d]$ we have $\sum_{i=1}^d \braket{\phi_i}{\phi^*_j} \ket{\phi_i} = \ket{\phi^*_j}$, where $\braket{\phi_i}{\phi^*_j}$ is the coordinate of $\phi^*_j$ with respect to the basis $\phi_i$. Thus, we can write 
    \begin{align}
        \Big ( \sum_{i=1}^d \ketbra{\phi_i}{\phi_i} \Big )T &=  \Big ( \sum_{i=1}^d \ketbra{\phi_i}{\phi_i} \Big ) \Big ( \sum_{j=1}^\infty \lambda_j \ketbra{\phi^*_j}{\phi^*_j}  \Big )  \\
        &= \sum_{j=1}^\infty \lambda_j \Big ( \sum_{i=1}^d \ketbra{\phi_i}{ \braket{\phi_i}{\phi^*_j} \phi^*_j} \Big )\\
        &= \sum_{j=1}^\infty \lambda_j \Big (  \sum_{i=1}^d   \braket{\phi_i}{\phi^*_j} \ketbra{ \phi_i}{ \phi^*_j} \Big )\\
        &= \sum_{j=1}^\infty \lambda_j \ketbra{\phi^*_j}{\phi^*_j}\\
        &= T \ .
    \end{align}
    Substituting this equality into \Cref{eq:rayleigh-ritz-aux1}, we obtain
    \begin{align}
        T \Big ( \sum_{i=1}^d {y}_{\ell_i} \ket{\phi_i} \Big ) = \lambda_{\ell} \Big ( \sum_{i=1}^d {y}_{\ell_i} \ket{\phi_i} \Big ).
    \end{align}
    Therefore, $\sum_{i=1}^d {y}_{\ell_i} \ket{\phi_i} = \ket {\phi^*_{\ell}}$ is an eigenfunction with the corresponding eigenvalue $\lambda_{\ell}$.
\end{proof}

\section{Proofs and Discussion on Low-rank Approximation}\label{app:lora}

\subsection{\texorpdfstring{Proof of~\Cref{thm:lora-eigenspace-extractor}}{Proof of Theorem 6}}\label{app:proof-lora-eigenspace-extractor}

\paragraph{\Cref{thm:lora-eigenspace-extractor}.}
{\it
    The LoRA objective in \Cref{eq:lora-base} is an eigenspace extractor for $T$, \ie $\sspan(\hat\phi_1,\dots,\hat\phi_d) = \sspan(\phi^*_1,\dots,\phi^*_d)$ and $\sspan(\hat\psi_1,\dots,\hat\psi_d) = \sspan(\psi^*_1,\dots,\psi^*_d)$. Furthermore, it admits orthogonal nested minimizers, such that $\hat \phi_i = a_i \phi^*_i$ and $\hat \psi_i = b_i \psi^*_i$ for some nonzero scalars $a_i, b_i$ where $a_i b_i = s_i$.}

\begin{proof}
    Let $(\hat\Phi, \hat\Psi)$ be a global minimizer of the objective in \Cref{eq:lora-base}. \Cref{thm:SEYM}(i) says that $(\hat\Phi, \hat\Psi)$ satisfy
    \begin{equation}
        \sum_{i=1}^{d} \ketbra{\hat\phi_i}{\hat\psi_i} =  \sum_{i=1}^{d}  s_i  \ketbra{\phi^*_i}{\psi^*_i}  \ .
    \end{equation}
    Denote the left and right-hand sides by the operators $L_d \coloneq \sum_{i=1}^{d} \ketbra{\hat\phi_i}{\hat\psi_i}$ and $T_d \coloneq \sum_{i=1}^{d} s_i \ketbra{\phi^*_i}{\psi^*_i}$. Since the two operators are identical, their ranges are also identical. For any $\psi \in \mcH_1$, we have
    \begin{equation}
        T_d \psi = \sum_{i=1}^{d} s_i \braket{\psi}{\psi^*_i} \phi^*_i \ ,
    \end{equation}
    which implies that $\range(T_d) \subseteq \sspan(\phi^*_1,\dots,\phi^*_d)$. Furthermore, we observe that any member of $\sspan(\phi^*_1, \ldots, \phi^*_d)$ can be written as $\sum_{i=1}^{d} c_i \phi^*_i$. By setting $\psi = \sum_{i=1}^{d} \frac{c_i}{s_i} \psi^*_i$ and using the orthonormality of $(\psi^*_i)_{i \in [d]}$, we observe $T_d \psi = \sum_{i=1}^{d} c_i \phi^*_i$ for any choice of $c \in \R^d$. So, $\sspan(\phi^*_1,\dots,\phi^*_d) \subseteq \range(T_d) $. Thus, $\range(T_d) = \sspan(\phi^*_1,\dots,\phi^*_d)$. Next, for any $\psi \in \mcH_1$, we have
    \begin{equation}
        L_d \psi = \sum_{i=1}^{d} \braket{\psi}{\hat\psi_i} \hat\phi_i \ ,
    \end{equation}
    which implies that $\range(L_d) = \range(T_d) \subseteq \sspan(\hat\phi_1,\dots,\hat\phi_d)$, and $\sspan(\phi^*_1,\dots,\phi^*_d) \subseteq \sspan(\hat\phi_1,\dots,\hat\phi_d)$. This implies that $\sspan(\hat\phi_1,\dots,\hat\phi_d)$ has at least dimension $d$, which means that $\hat\phi_1,\dots,\hat\phi_d$ must be linearly independent functions, and in fact $\sspan(\hat\phi_1,\dots,\hat\phi_d)$ has exactly dimension $d$. Thus, we conclude that $\sspan(\hat\phi_1,\dots,\hat\phi_d) = \sspan(\phi^*_1,\dots,\phi^*_d)$. Similarly, we can show that $\sspan(\hat\psi_1,\dots,\hat\psi_d) = \sspan(\psi^*_1,\dots,\psi^*_d)$, which concludes the proof that LoRA objective is an eigenspace extractor (or more precisely, \emph{singular-space} extractor) for~$T$.

    Next, we show that this problem has orthogonal nested minimizers. We will prove this claim by induction. For $d=1$, the eigenspace extractor property immediately implies that $\hat\phi_1 = \alpha_1 \phi^*_1$ and $\hat\psi_1 = \beta_1 \psi^*_1$ for some constants $\alpha_1, \beta_1$ such that $\ketbra{\hat\phi_1}{\hat\psi_1} = s_1 \ketbra{\phi^*_1}{\psi^*_1}$. As induction hypothesis, assume that $\hat\phi_i = \alpha_i \phi^*_i$ and $\hat\psi_i = \beta_i \psi^*_i$, $\forall i \in [d-1]$ where $\ketbra{\hat\phi_i}{\hat\psi_i} = s_i \ketbra{\phi^*_i}{\psi^*_i}$.
    Then, if we define $T_{d-1} \defeq \sum_{i=1}^{d-1} s_i \ketbra{\phi^*_i}{\psi^*_i}$, the problem in \Cref{eq:lora-base} becomes
    \begin{equation}\label{eq:lora-schmidt-marginal}
        \argmin_{\psi_d \in \mcH_1, \ \phi_d \in \mcH_2} \norm[\Big]{T - T_{d-1} -  \ketbra{\phi_d}{\psi_d}}_{\rm HS}^2 \ .
    \end{equation}
    \Cref{thm:SEYM}(i) implies that the solution of this satisfies $\ketbra{\hat\phi_d}{\hat\psi_d} = s_d \ketbra{\phi^*_d}{\psi^*_d}$. Consequently, from the eigenspace extractor result above applied to the operator $T-T_{d-1}$ (which has $\phi^*_d$ and $\psi^*_d$ as its top left/right singular functions), we have $\range(\ketbra{\hat\phi_d}{\hat\psi_d})=\sspan(\phi_d^*)$ which implies that $\hat\phi_d = \alpha_d \phi^*_d$ for some constant $\alpha_d$. Similarly, $\hat\psi_d = \beta_d \psi^*_d$ for some constant $\beta_d$ such that $\alpha_d \beta_d = s_d$. Thus, the induction step is complete, and the problem admits orthogonal nested minimizers since singular functions $(\phi^*_i)_{i \geq 1}$ (and $(\psi^*_i)_{i \geq 1}$) are pairwise orthogonal. 
\end{proof}

\subsection{Proof of LoRA Objective Simplification}\label{app:proof-lora-error-HS}

In \Cref{sec:lora}, we have defined the low-rank approximation (LoRA) objective $\mcL_d \coloneq \mcL_{\lora}(\bphi_{[d]}, \bpsi_{[d]})$ as 
\begin{equation}\label{eq:lora-appendix}
    \mcL_d \coloneq -2\sum_{i=1}^{d} \braket{\phi_i}{T \psi_i}_{\mcH_2} + \sum_{i=1}^{d} \sum_{j=1}^{d} \braket{\phi_i}{\phi_j}_{\mcH_2} \braket{\psi_i}{\psi_j}_{\mcH_1} \ .
\end{equation}
Minimizing $\mcL_d$ is equivalent to minimizing the original objective in \Cref{eq:lora-base} since the two objectives only differ by the constant $\norm[\big]{T}_{\rm HS}^2$, shown next.

\begin{lemma}\citep[Lemma C.1]{ryu2024operator} 
If $T$ is a Hilbert-Schmidt operator, then 
    \begin{equation}\label{eq:lora-error-HS}
        \mcL_{\lora}(\bphi_{[d]}, \bpsi_{[d]}) = \norm[\Big]{T - \sum_{i=1}^{d} \ketbra{\phi_i}{\psi_i}}_{\rm HS}^2 - \norm[\big]{T}_{\rm HS}^2 \ .
    \end{equation}
\end{lemma}
\begin{proof}
    Let $\{h_j : j \in J\}$ be an orthonormal basis of $\mcH_1$. Also note that $\sum_{j \in J} \ketbra{h_j}{h_j} = I$, where $I$ denotes the identity operator. Then, we have
    \begin{equation}
        \begin{aligned}
            \norm[\Big]{T - \sum_{i=1}^{d} \ketbra{\phi_i}{\psi_i}}_{\rm HS}^2 &- \norm[\big]{T}_{\rm HS}^2 \\ 
            &= \sum_{j \in J} \norm[\Big]{ T \ket{h_j} - \sum_{i=1}^{d} \ket{\phi_i} \braket{\psi_i}{h_j}}^2  - \sum_{j \in J} \norm{ T \ket{h_j}}^2 \\
            &= \sum_{j \in J} \left( -2 \sum_{i=1}^{d} \braket{\psi_i}{h_j} \braket{\phi_i}{Th_j} + \sum_{i=1}^{d} \sum_{i'=1}^{d} \braket{\psi_i}{h_j} \braket{\psi_{i'}}{h_j} \braket{\phi_i}{\phi_{i'}} \right) \\
            &= -2 \sum_{i=1}^{d} \braket{\phi_i}{T \Big(\sum_{j \in J} \ketbra{h_j}{h_j}\Big)\psi_i} + \sum_{i=1}^{d} \sum_{i'=1}^{d} \braket{\psi_i}{\Big(\sum_{j \in J} \ketbra{h_j}{h_j}\Big)\psi_{i'}} \braket{\phi_i}{\phi_{i'}} \\
            &= -2 \sum_{i=1}^{d} \braket{\phi_i}{T \psi_i} + \sum_{i=1}^{d} \sum_{i'=1}^{d} \braket{\psi_i}{\psi_{i'}} \braket{\phi_i}{\phi_{i'}} \\
            &= \mcL_{\lora}(\bphi_{[d]}, \bpsi_{[d]}) \ .
        \end{aligned}
    \end{equation}
\end{proof}

\subsection{LoRA for Contextual Kernel}\label{app:lora-contextual-kernel}

We give more detailed derivations for \Cref{sec:lora-contextual-kernel} on the contextual kernel $\kxa = \frac{P^+(x,a)}{P_X(x)P_A(a)}$. The first term of the LoRA objective in \Cref{eq:lora} becomes
\begin{align}
    \braket{\phi_i}{ \txa \psi_i} =  \int_{x,a} \phi_i(x) \frac{P^+(x,a)}{\px(x)\pa(a)} \psi_i(a)  \px(x)\pa(a) dx da = \E_{x,a \sim P^+(x,a)}[ \phi_i(x) \psi_i(a)] \ .
\end{align}
The second term of the LoRA objective becomes
\begin{align}
    \sum_{i=1}^{d} \sum_{j=1}^{d}  \braket{\phi_i}{\phi_j} \braket{\psi_i}{\psi_j} &= \sum_{i=1}^{d} \sum_{j=1}^{d}  \int_{a} \psi_i(a) \psi_j(a) \pa(a) \, da \int_{x} \phi_i(x) \phi_j(x) \px(x) \, dx \\
    &= \int_{x} \int_{a} \sum_{i=1}^{d} \sum_{j=1}^{d} \phi_i(x)\psi_i(a)  \phi_j(x) \psi_j(a) \px(x) \pa(a) \, dx\, da \\ 
    &= \int_{x} \int_{a} \Big(\sum_{i=1}^{d}\phi_i(x)\psi_i(a) \Big) \Big( \sum_{j=1}^{d}   \phi_j(x) \psi_j(a) \Big) \px(x) \pa(a) \, dx \, da \\
    &= \E_{x \sim \px, a \sim \pa} \left[ \big(\Phi(x)^\top \Psi(a) \big)^2 \right] \ .
\end{align}
Thus, the final variational objective for the contextual kernel is given as follows
\begin{equation}
    \mcL_{\lora}(\Phi, \Psi) = -2 \E_{x,a \sim P(x,a)}[ \Phi(x)^\top \Psi(a)] 
 + \E_{x \sim \px, a \sim \pa} \left[ \big(\Phi(x)^\top \Psi(a) \big)^2 \right]  \ . 
\end{equation}
The pseudocode for training LoRA for the contextual kernel is given in \Cref{alg:lora}. 

\begin{algorithm}[t]
    \caption{LoRA for eigenspace extraction of the contextual kernel}
    \label{alg:lora}
    \begin{algorithmic}
    \State \textbf{Require:} $X \sim \px$, $A \sim P^+(\cdot \mid X)$, batch size $m$ 
    \State \textbf{Output:} $\Phi = [\phi_1, \dots, \phi_d] : \mcX \to \mathbb{R}^d$, and $\Psi = [\psi_1, \dots, \psi_d] : \mcA \to \mathbb{R}^d$
    \State Initialize $\Phi: \mcX \to \R^d$, $\Psi \to \R^d$
    \For{each training step}
        \State Sample batch of inputs $\{x_1,\dots,x_{m}\}$, then context samples $a_i \sim P^+(\cdot | x_i)$ for $i\in[m]$
        \State Estimate the first term in \Cref{eq:lora-contextual-final} using all pairs $\{(x_i,a_i): i \in [m]\}$ in the minibatch:
        \[ \hat \E_{(x,a) \sim P^+(x,a)} [\Phi(x)^{\top}\Psi(a)] \approx \frac{1}{m}\sum_{i \in [m]} \Phi(x_i)^{\top}\Psi(a_i) \]
        \State Estimate the second term in \Cref{eq:lora-contextual-final} using all $m(m-1)$ pairs $(x_i,a_{(-i)})$ in the minibatch
        \[ \hat \E_{x \sim \px, a \sim \pa} \left[ \big(\Phi(x)^\top \Psi(a) \big)^2 \right] \approx \frac{1}{m} \sum_{i\in[m]} \frac{1}{(m-1)} \sum_{j \in [m] \setminus \{i\}} (\Phi(x_i)^{\top}\Psi(a_j))^2 \]
        \State Compute gradient of the loss $\mcL_{\lora}(\Phi, \Psi)$ and update the parameters of $\Phi$ and $\Psi$  
    \EndFor
    \end{algorithmic}
\end{algorithm}

\subsection{Derivation of SCL-LoRA Equivalence}\label{app:contrastive-lora-equivalence}

We expand on the discussion in \Cref{sec:lora-contextual-kernel} on contrastive learning and the LoRA connection. 
Spectral contrastive loss (SCL)~\citep{haochen2021provable} is given by
\begin{equation}\label{eq:contrastive-loss-appendix}
   \mcO_{\rm SCL} : \min_{\psi_1,\dots, \psi_d \in \lap}  \underset{X \sim \px}{\E} \; \underset{A,A^+ \sim P^+(\cdot |X)}{\E} \bigg[ - \dotp{\tPsi(A), \tPsi(A^+)} + \frac{1}{2} \underset{A^- \sim \pa}{\E} \Big[\dotp{\tPsi(A),\tPsi(A^-)}^2 \Big] \bigg]   ,
\end{equation}
where $\tPsi = \Psi - \E[\Psi]$ are centered features. The next result shows that $\mcO_{\rm SCL}$ extracts the eigenspace.

\begin{theorem}\label{thm:contrastive}
[Theorem 3.4 in \cite{zhai2025contextures}]\label{thm:SCL}    
Denote the minimizer of the contrastive objective $\mcO_{\rm SCL}$ in \Cref{eq:contrastive-loss-appendix} by $\htPsi$. Then, $\htPsi$ extracts the top-$d$ eigenspace of $\taa$, and $\htPhi = \txa\htPsi$ extracts the top-$d$ eigenspace of $\txx$.
\end{theorem}

\begin{proof}
    For completeness, we give the proof of this result. Suppose $\psi_i = \sum_{j\geq 0}c_{ij}\psi^*_j$ where $(\psi_j)_{j \geq 0}$ form an orthonormal basis for $\lap$. Extracting the trivial top-eigenfunction, let $\tpsi_i = \psi_i - \E[\psi_i] = \sum_{j\geq 1} c_{ij}\psi^*_j$. Denote matrix $\bC$ with entries  $(c_{ij})_{i, j \in [d]}$, matrix $\bB \coloneq \bC^\top \bC$ with entries $(b_{ij})_{i,j\in[d]}$, and let $s_i$ be the $i^\textrm{th}$ singular value of $\txa$ and $\lambda_i =s_i^2$. Then, we have 
    \begin{align}
        \underset{X \sim \px}{\E} & \underset{A,A' \sim P^+(\cdot |X)}{\E} \brac{-\dotp{\tPsi(A), \tPsi(A')}} \\
        &= - \iiint \dotp{\tPsi(a), \tPsi(a')} P^+(a|x) P^+(a'|x) \px(x) dx\,da\,da' \\
        &= - \int \dotp{\int \tPsi(a) P^+(a|x) dy, \int \tPsi(a') P^+(a'|x) da'} \px(x) dx\\
        &= - \int \dotp{\txa \tPsi(x), \txa \tPsi(x)} \px(x) dx = \|\txa \tPsi\|^2_{\px} \\ 
        &= - \sum_{i} \lambda_i b_{ii} \ , \label{eq:contrastive-expansion-1}
    \end{align}     
    where in the last line we have used $\txa \psi^*_j = s_i \phi_j^*$. Next, for the second term of SCL, we have
\begin{align}
    \underset{A,A^- \sim \pa}{\E} \brac{\dotp{\tPsi(A), \tPsi(A^-)}^2}
    &= \iint \Big[\sum_{i=1}^d \tilde \psi_i(a)\tilde \psi_i(a^-)\Big]^2 d\pa(a) d\pa(a^-)  \\
    &= \sum_{1\leq i,j\leq d} \brac{\int \tilde\psi_i(a)\tilde\psi_j(a) d\pa(a)}^2  \label{eq:contrastive-expansion-2} = \sum_{i,j} b_{ij}^2   .
\end{align}
Thus, SCL becomes
\begin{equation}
   - \sum_i \lambda_i b_{ii} + \frac12 \sum_{i,j} b_{ij}^2 \\= \|\bB - \bD\|_{\rm F}^2 - \|\bD\|_{\rm F}^2  \ ,
\end{equation}
and if suffices to minimize $\|\bB - \bD\|_{\rm F}^2$ where $\rank(\bB) \leq d$. By \Cref{thm:SEYM}(ii), we know that the minimizer is $\hat\bB = \text{diag}(\lambda_1,\cdots,\lambda_d)$, which gives $\hat\bC = \bU\diag(s_1,\cdots,s_d)$ where $\bU \in \R^{d\times d}$ is an orthonormal matrix. This means that the minimizer $\htPsi = \hat \bC \Psi^*$ extracts the linear span of $(\psi^*_i)_{i \in [d]}$, \ie top-$d$ eigenspace of $\taa$. Equivalently, $\htPhi = \txa\htPsi$ extracts the top-$d$ eigenspace of $\txx$. 
\end{proof}

Now, recall the LoRA objective in \Cref{eq:lora}. For simplicity, write it for the eigenfunctions of a self-adjoint operator (as opposed to the general SVD):
\begin{equation}\label{eq:lora-simplified}
   \min_{\tpsi_1,\dots,\tpsi_d \in \lap} -2\sum_{i=1}^{d} \braket{\tpsi_i}{\txa\tpsi_i} + \sum_{i=1}^{d} \sum_{j=1}^{d} \braket{\psi_i}{\psi_j}^2 \ . 
\end{equation}
Note that \Cref{eq:contrastive-expansion-2} implies that the second terms of the LoRA objective~\Cref{eq:lora-simplified} and contrastive objective~\Cref{eq:contrastive-loss-appendix} are equivalent. Also, we have
\begin{align}
    \tpsi_i &= \sum_{j} c_{ij} \psi^*_j \quad \text{and} \quad
    \txa\tpsi_i = \sum_{j} c_{ij} \txa\psi^*_j = \sum_{j} c_{ij} \lambda_j \psi^*_j \ .
\end{align}
Using the fact that $\lambda_j = s_j^2$ and the definition of matrix $\bB$, we obtain $\braket{\tpsi_i}{\txa\tpsi_i} = \sum_{i} s_{i}^2 b_{ii}$. Therefore, using \Cref{eq:contrastive-expansion-1}, we see that the first terms of the LoRA and contrastive objectives are equivalent. Hence, the spectral contrastive loss is a special case of the unconstrained LoRA objective for extracting the top-$d$ eigenspace. \looseness-1

\subsection{Connection to Related Work on Low-rank Approximation}\label{app:related-lora}

To our knowledge, \cite{bengio2004learning} is the first work to study eigenfunction extraction from an optimization angle. Specifically, they proposed a method for sequentially extracting the eigenpairs of a continuous kernel $K$ via minimizing the factorization error, that is, the approximation error of $K$ by the rank-$d$ approximation $K_d$:\looseness=-1
\begin{equation}
    K_d(x,y) \coloneq \sum_{i=1}^m \lambda_i^* \phi_i^*(x) \phi_i^*(y) \ .
\end{equation}
Given the eigenpairs $\{(\lambda_i^*, \phi_i^*)\}_{i=1}^{d-1}$, the $d$-th eigenpair $(\lambda_d^*, \phi_d^*)$ can be obtained by minimizing the following objective
\begin{equation}\label{eq:bengio-sequential}
    \min_{\lambda_d, \phi_d : \norm{\phi_d}_\px^2 = 1} \iint \big( K(x,y) - \lambda_d \phi_i(x) \phi_i(y) - K_{d-1}(x,y) \big)^2 dP_{\mcX}(x)dP_{\mcX}(y) \ .
\end{equation}
\begin{proposition}\citep[Proposition 3]{bengio2004learning}\label{thm:bengio-sequential}
    Sequentially solving \Cref{eq:bengio-sequential} for $i = 1,\dots, d$ gives the top-$d$ eigenpairs $\{(\lambda_i^*, \phi_i^*)\}_{i \in [d]}$ of $K$.
\end{proposition}
Although this sequential approach gives exact eigenfunctions and resembles the sequential nesting method in \Cref{sec:sequential-nesting}, it has the obvious limitation of failing the joint optimization criteria. The original paper \citep{bengio2004learning} did not discuss how to recover the eigenfunctions (or the eigenspace) jointly. However, we can transform this into a joint optimization problem by trading the precision of the top-$d$ eigenpairs for retaining only the top-$d$ eigenspace. Then, having access to an \emph{eigenspace extractor}, we can follow the methods in \Cref{sec:eigenspace-to-eigenfunctions} to obtain the exact eigenfunctions via joint optimization.

Let $T$ and $T_d$ be integral operators of $K$ and $K_d$. Then, Hilbert-Schmidt norm of $T - T_d$ is given by
\begin{equation}
    \norm{T - T_d}_{\rm HS}^2 = \int \int | K(x,y) - K_d(x,y)|^2 dP_{\mcX}(x)dP_{\mcX}(y)  \ .
\end{equation}
For $\Phi = (\phi_1,\dots,\phi_d) \in \lxp^d$ and $\lambda \in \R^d$, denote $K_{\lambda, \Phi} \coloneq \sum_{i=1}^{d} \lambda_i \phi_i(x) \phi_i(y)$ with operator $T_{\lambda, \Phi}$. Then, the joint extension of \Cref{eq:bengio-sequential} becomes the following optimization problem
\begin{equation}\label{eq:bengio-joint}
    \min_{\lambda, \Phi} \;\; \norm{T - T_{\lambda, \Phi}}_{\rm HS}^2 \ . 
\end{equation}
The above derivation can be viewed as a special case of \Cref{thm:lora-eigenspace-extractor}, where we demonstrated that the base LoRA objective in \Cref{eq:lora-base} is, in fact, an eigenspace extractor for the operator $T$.

Recently, \citet{ryu2024operator} proposed using nesting methods with the LoRA objective for exact singular function extraction, called NeuralSVD. Recall that our \Cref{thm:sequential-nesting-general} (sequential nesting) and \Cref{thm:joint-nesting-general} (joint nesting) show the extraction of ordered eigenfunctions given \emph{any} base problem $\mcO$ that is an eigenspace extractor and admits orthogonal nested minimizers. Since LoRA objective is shown to be such a base problem (\Cref{thm:lora-eigenspace-extractor}), the following results of \citet{ryu2024operator} become special cases of \Cref{thm:sequential-nesting-general} and \Cref{thm:joint-nesting-general}. 
 
\begin{proposition}\citep[Theorem 3.2, Sequential Nesting]{ryu2024operator}\label{thm:lora-sequential}
    Suppose that $T: \mcH_1 \to \mcH_2$ is compact. Let $(\hat\psi_m, \hat\phi_m) \in \mcH_1 \times \mcH_2$ be a global minimizer of $\mcL_{\lora}(\bpsi_{[m]}, \bphi_{[m]})$, where $\sum_{i=1}^{m-1} \ketbra{\phi_i}{\psi_i} = \sum_{i=1}^{m-1} s_i \ketbra{\phi^*_i}{\psi^*_i}$. If $s_m > s_{m+1}$, then $\ketbra{\hat\phi_m}{\hat\psi_m} = s_m \ketbra{\phi^*_m}{\psi^*_m}$.
\end{proposition}

\begin{proposition}\citep[Theorem 3.3, Joint Nesting]{ryu2024operator}\label{thm:lora-joint}
    Assume that $T: \mcH_1 \to \mcH_2$ is compact and top-$(d+1)$ singular values are distinct. Let $\{(\hat\psi_i, \hat\phi_i) : i \in [d]\} \in (\mcH_1 \times \mcH_2)^d$ be the global minimizer of 
    \begin{equation}\label{eq:neuralsvd-joint-objective}
        \mcL_{\jnt}(\bpsi_{[d]},\bphi_{[d]},\bw) \coloneq \sum_{i=1}^{d} w_i \mcL_{\lora}(\bpsi_{[i]}, \bphi_{[i]})
    \end{equation}
    for some positive weights $\bw \in \R_+^d$. Then, $\ketbra{\hat\phi_i}{\hat\psi_i} = s_i \ketbra{\phi^*_i}{\psi^*_i}$ for each $i\in[d]$.
\end{proposition}

\section{Proofs and Discussion on Rayleigh Quotient Optimization}\label{app:rayleigh}

\subsection{\texorpdfstring{Proof of~\Cref{thm:rq-eigenspace-extractor}}{Proof of Theorem 7}}\label{app:proof-rq-eigenspace-extractor}

\paragraph{\Cref{thm:rq-eigenspace-extractor}.}
{\it
The RQ objective in \Cref{eq:rq-svd} is an eigenspace extractor for $T$, \ie
$\sspan(\hat\phi_1,\dots,\hat\phi_d) = \sspan(\phi^*_1,\dots,\phi^*_d)$ and $\sspan(\hat\psi_1,\dots,\hat\psi_d) = \sspan(\psi^*_1,\dots,\psi^*_d)$. Furthermore, it admits orthogonal nested minimizers, $\hat \phi_i = c_i \phi^*_i$ and $\hat \psi_i = c_i \psi^*_i$ for $c_i = \pm 1$.}

\begin{proof}
    Since $(\phi^*_j)_{j \geq 1}$ and $(\psi^*_j)_{j \geq 1}$ form orthonormal bases for $\mcH_2$ and $\mcH_1$, respectively, let us denote
    \begin{equation}
        \phi_i = \sum_{j\geq 1} \alpha_{ij}\phi^*_j \quad \mbox{and} \quad \psi_i = \sum_{j \geq 1} \beta_{ij}\psi^*_j \ ,
    \end{equation}
    where $\balpha_i = (\alpha_{i1},\dots)$ and $\bbeta_i = (\beta_{i1},\dots)$ are coefficient vectors for $i \in [d]$. Then, substituting $T \psi^*_j = s_j \phi^*_j$, the objective in \Cref{eq:rq-svd} becomes
    \begin{align}
        -\sum_{i=1}^{d} \braket{\phi_i}{T \psi_i}_{\mcH_2} & = -\sum_{i=1}^{d} \braket{\sum_{j\geq 1}\alpha_{ij}\phi^*_j}{\sum_{j\geq 1}\beta_{ij} s_j \phi^*_j}_{\mcH_2}  = -\sum_{i=1}^{d} \sum_{j\geq 1} \alpha_{ij} \beta_{ij} s_j
    \end{align}
    where in the last equality, we have used the orthonormality of the true singular functions $(\psi^*_j)_{j \geq 1}$. Thus, the objective is equivalent to 
    \begin{equation}
        \maximize_{\{\balpha_i, \bbeta_i \; : \; i \in [d]\}} \sum_{j\geq 1} s_j  \sum_{i=1}^{d} \alpha_{ij}\beta_{ij} \ .
    \end{equation}
    Using the arithmetic-geometric mean inequality, 
    \begin{equation}
        \sum_{j\geq 1} s_j  \sum_{i=1}^{d} \alpha_{ij}\beta_{ij} \leq \sum_{j \geq 1} s_j \sum_{i=1}^{d}\frac{\alpha_{ij}^2+\beta_{ij}^2}{2} = \sum_{j\geq 1} s_j c_j \ ,
    \end{equation}
    where we defined $c_j \coloneq \sum_{i=1}^{d}\frac{\alpha_{ij}^2+\beta_{ij}^2}{2}$. Next, note that the constraints for $(\phi_i)_{i \in [d]}$ and $(\psi_j)_{j \in [d]}$ in \Cref{eq:rq-svd} and the orthonormality of the true singular functions imply that $\norm{\balpha_i}_2 = \norm{\bbeta_i}_2 = 1$ for all $i\in [d]$, $\balpha_i^\top \balpha_j = 0$ and $\bbeta_i^\top \bbeta_j = 0$ for all $i \neq j$. Therefore,
    \begin{equation}\label{eq:bound-cj-1}
        \sum_{j \geq 1}\sum_{i=1}^{d} \alpha_{ij}^2 = d \ , \;\; \sum_{j \geq 1}\sum_{i=1}^{d} \beta_{ij}^2 = d \;\; \implies \sum_{j\geq 1} c_j = d \ .
    \end{equation}
    Moreover, let $\be_j= (0,\cdots,0,1, 0, \cdots)$ denote the $j$-th standard basis vector, whose $j$-th entry is one.
    We have 
    \begin{equation}\label{eq:bound-cj-2}
        \sum_{i=1}^{d} \alpha_{ij}^2 = \sum_{i=1}^{d} \langle \balpha_i, \be_j \rangle^2 \leq \norm{\be_j}_2^2 = 1,
    \end{equation}
    where we apply Bessel's inequality to the orthonormal sequence $\{\balpha_1, \balpha_2,\cdots, \balpha_d \}$. Similarly, we also have $\sum_{i=1}^{d} \beta_{ij}^2 \leq 1$. Together, they imply $c_j \leq 1$ for all $j$.
    
    Since singular values $(s_j)_{j \geq 1}$ are in non-decreasing order, the objective $\sum_{j\geq 1}s_j c_j$ attains its maximum when $c_j=1$ for all $j\in[d]$ and $c_j=0$ for all $j>d$, \ie $\alpha_{ij}=\beta_{ij}=0$ for all $j>d$. Furthermore, for the equality of the arithmetic-geometric inequality, we also have $\alpha_{ij}=\beta_{ij}$ for all $i,j \in [d]$. Therefore, the minimizers of \Cref{eq:rq-svd} become
    \begin{equation}
        \hat \phi_i = \sum_{j=1}^{d} \alpha_{ij}\phi^*_j \, \quad \mbox{and} \quad  \hat \psi_i = \sum_{j=1}^{d} \alpha_{ij}\psi^*_j \ , \quad \forall i \in [d] \ . 
    \end{equation}
    Also noting that $(\balpha_i)_{i \in [d]}$ vectors are orthonormal, we conclude that $\sspan(\hat\phi_1,\dots,\hat\phi_d)=\sspan(\phi^*_1,\dots,\phi^*_d)$ and $\sspan(\hat\psi_1,\dots,\hat\psi_d)=\sspan(\psi^*_1,\dots,\psi^*_d)$, which completes the proof that the objective is an eigenspace extractor.

    We will prove the second part by induction. For $d=1$, $\sspan(\hat\phi_1)=\sspan(\phi^*_1)$ and the constraint $\braket{\hat\phi_1}{\hat\phi_1}=1$ together yield that $\hat\phi_1 = \pm \phi^*_1$. Similarly, we have $\hat\psi_1=\pm \psi^*_1$. As induction hypothesis, assume that for $\mcO_1,\dots,\mcO_{d-1}$, the problem in \Cref{eq:rq-svd} has orthogonal nested minimizers $\hat \phi_i = \pm \phi^*_i$ and $\hat \psi_i = \pm \psi^*_i$ for $i\in[d-1]$. Then, let
    \begin{equation}\label{eq:induction-step-rq-svd}
        (\hat \phi_d, \hat \psi_d) \in \argmax_{\psi_d \in \mcH_1, \phi_d \in \mcH_2} \braket{\phi_d}{T \psi_d}_{\mcH_2} \quad \st \;\;
        \begin{aligned}
        \braket{\phi_d}{\phi_j}_{\mcH_2} &= \delta_{ij} \\
        \braket{\psi_d}{\psi_j}_{\mcH_1} &= \delta_{ij}
    \end{aligned}  \ , \;\; \forall j \in [d] \ .
    \end{equation}
    Note that $\hat\phi_d$ must be in $\sspan(\phi^*_1,\dots,\phi^*_d)$ since $\mcO_d$ is an eigenspace extractor. Similarly $\hat \psi_d \in \sspan(\psi^*_1,\dots,\psi^*_d)$. Then, the orthonormality constraints in \Cref{eq:induction-step-rq-svd} yield that $\hat\phi_d = \pm \phi^*_d$ and $\hat \psi_d = \pm \psi^*_d$, which completes the induction step and the proof of the theorem.
\end{proof}

\subsection{\texorpdfstring{Proof of~\Cref{thm:rq-direct-eigenfunctions}}{Proof of Theorem 8}}\label{app:proof-rq-direct-eigenfunctions}

\paragraph{\Cref{thm:rq-direct-eigenfunctions}}
{\it 
    The following objective (\Cref{eq:rq-direct-eigenfunctions-appendix}) obtains a permutation of the true eigenfunctions directly, \ie the minimizers $(\hat\phi_i)_{i\in [d]}$ satisfy $\hat \phi_i = \pm \phi^*_{\sigma(i)}$ for some permutation $\sigma:[d]\mapsto[d]$.
    \begin{equation}\label{eq:rq-direct-eigenfunctions-appendix}
        \min_{\substack{\phi_i \in \mcH}} - \sum_{i=1}^{d} \braket{\phi_i}{ T  \phi_i}_{\mcH}\;\;\st\;\;
        \begin{aligned}
            \braket{\phi_i}{\phi_i}_{\mcH} &= 1, \; \forall i. \\
            \braket{\phi_i}{T \phi_j}_{\mcH} &= 0, \;\forall i \neq j. \\
        \end{aligned}
    \end{equation}
}

\begin{proof}
    Let compact, self-adjoint, positive definite operator $T$ has distinct eigenvalues $(\lambda_i)_{i \geq 1}$ in descending order, and therefore the eigenfunctions $(\phi^*_i)_{\geq i}$ form an orthonormal basis for $\mcH$.
    Denote $\phi_i = \sum_{j \geq 1} \alpha_{ij}\phi^*_j$, where $\balpha_i = (\alpha_{i1},\dots)$ is the coefficient vector for $i \in [d]$. 
    We will prove the desired result by showing that optimal $(\hat \phi_i)_{i \in[d]}$ will have $\balpha_i = \pm \be_{\sigma(i)}$ for a permutation $\sigma : [d] \mapsto [d]$ where $\be_i$ denotes the $i$-th standard basis vector $(0,\dots,0,1,0,\dots)$
    We start by rewriting the optimization problem. Substituting $T \phi^*_j = \lambda_j \phi^*_j$, the objective in \Cref{eq:rq-direct-eigenfunctions-appendix} becomes minimizing
    \begin{align}
        -\sum_{i=1}^{d} \braket{\phi_i}{T \phi_i}_{\mcH} &= - \sum_{i=1}^{d} \braket{\sum_{j \geq 1} \alpha_{ij}\phi^*_j}{\sum_{j \geq 1}\alpha_{ij}T \phi^*_j}_{\mcH} \\
        &= - \sum_{i=1}^{d} \braket{\sum_{j \geq 1}\alpha_{ij}\phi^*_j}{\sum_{j \geq 1}\alpha_{ij}\lambda_j \phi^*_j}_{\mcH} \\
        &= - \sum_{i=1}^{d} \sum_{j \geq 1} \alpha_{ij}^2 \lambda_j \\
        &= - \sum_{i=1}^{d} \sum_{j \geq 1} \beta_{ij}^2 \ , 
    \end{align}
    where $\beta_{ij} \coloneq \sqrt{\lambda_j}\alpha_{ij}$ and we have used the orthonormality of the true eigenfunctions $(\phi^*_j)_{j \geq 1}$, $\braket{\phi^*_j}{\phi^*_j}_{\mcH}=1$. Define the vectors $\bbeta_i = (\beta_{i1},\beta_{i2},\dots)$ for $i\in[d]$, for which the second constraint becomes $\langle \bbeta_i, \bbeta_j \rangle =0$ for all $i \neq j$, $i,j \in [d]$. Also, the first constraint becomes
    \begin{equation}
        1 = \braket{\phi_i}{\phi_i}_{\mcH} = \sum_{j \geq 1}\alpha_{ij}^2 = \sum_{j \geq 1}\frac{\beta_{ij}^2}{\lambda_j} \ .
    \end{equation}
    Then, the constrained optimization problem is equivalently given by
    \begin{equation}\label{eq:rq-direct-eigenfunctions-appendix-objective-2}
        \maximize_{\{\bbeta_i \; : \; i \in [d]\}} \sum_{i=1}^{d} \sum_{k \geq 1} \beta_{ik}^2  \quad \st \quad 
        \begin{aligned}
            &\sum_{k \geq 1} \frac{\beta_{ik}^2}{\lambda_k} = 1 \ , \; \forall i \in [d] \ . \\
            & \sum_{k \geq 1}\beta_{ik}\beta_{jk} = 0 \;\; \forall 1 \leq i < j \leq d \ .
        \end{aligned} 
    \end{equation}
    We now use Lagrange multipliers to infer the structure of the optimal solutions. Specifically, we will show that the optimal solutions $\bbeta_i$'s are all supported on at most $d$ coordinates only. 
    To show this, consider a convex relaxation of the problem where we include the interior of the ellipsoid in the feasible region, \ie we change the constraints $\sum_{k\geq 1} \frac{\beta_{ik}^2}{\lambda_k}=1$ to $\sum_{k\geq 1} \frac{\beta_{ik}^2}{\lambda_k} \leq 1,~\forall i \in [d]$. 
    Note that the optimal solution of the relaxed problem will still satisfy the original equality constraints since we are maximizing a positive quadratic form, and scaling up any $\bbeta_i$ that does not satisfy the equality constraint will increase the objective while still satisfying the relaxed constraint. Thus, the optimal value and the minimizer of the relaxed problem are equal to those of the original problem in \Cref{eq:rq-direct-eigenfunctions-appendix-objective-2}.
    
    So, we write the Lagrangian for this relaxed problem with non-negative Lagrange multipliers $\mu_1, \ldots, \mu_d > 0$ and $m_{ij} \in \R$ for all $1 \leq i < j \leq d$ as follows: 
    \begin{equation}
        \mcL\Big((\bbeta_i)_{i \in [d]}, (\mu_i)_{i \in [d]}, (m_{ij})_{1\leq i < j \leq d}\Big) = \sum_{i=1}^d \sum_{k\geq 1} \beta_{ik}^2 - \sum_{i=1}^d \mu_i \left ( \sum_{k\geq 1} \frac{\beta_{ik}^2}{\lambda_k} - 1 \right ) - \sum_{\substack{i, j \in [d] \\ i < j}} m_{ij} \Big( \sum_{k \geq 1}\beta_{ik}\beta_{jk} \Big)
        \ .
    \end{equation}
    We also define $m_{ji} = m_{ij}$ for all $i<j$.
    Here, we want to maximize the Lagrangian $\mcL$ and use KKT conditions. To find the stationary points, we set the derivative with respect to each $\beta_{ik}$ to zero,
    \begin{equation} \label{eqn:kkt-condition}
        \frac{\partial \mcL}{\partial \beta_{ik}} = 2\beta_{ik} \left(1 - \frac{\mu_i}{\lambda_k} \right) - \sum_{j \in [d] \setminus \{i\}} m_{ij}\beta_{jk} = 0 \ , \quad \forall  i \in [d], \, k \geq 1 \ .
    \end{equation}
    Multiplying both sides by $\beta_{ik}$, and summing over $k \geq 1$, we obtain 
    \begin{align} \label{eqn:kkt-condition1}
         2 \sum_{k \geq 1} \beta_{ik}^2\left(1 - \frac{\mu_i}{\lambda_k} \right) - \sum_{j \in [d] \setminus \{i\}} m_{ij} \sum_{k \geq 1} \beta_{ik}\beta_{jk} = 0 \ , \quad \forall i \in [d] \ .
    \end{align}
    Using the constraints $\sum_{k \geq 1} \beta_{ik}\beta_{jk} = 0,~\forall i \neq j$ in \Cref{eq:rq-direct-eigenfunctions-appendix-objective-2}, the second term in \Cref{eqn:kkt-condition1} vanishes. Using the first constraint $\sum_{k \geq 1} \frac{\beta_{ik}^2}{\lambda_k} = 1$, we obtain
    \begin{equation}
        \mu_i = \sum_{k\geq 1} \beta_{ik}^2, \quad \forall  i \in [d] \ .
    \end{equation}
    Next, multiplying both sides of \Cref{eqn:kkt-condition} by $\beta_{\ell k}$ for some $\ell \in [d] \setminus \{i\}$ and summing over $k \geq 1$, we obtain
    \begin{align} \label{eqn:kkt-condition2}
         & 2 \sum_{k \geq 1} \beta_{ik} \beta_{\ell k}\left(1 - \frac{\mu_i}{\lambda_k} \right) - \sum_{j \in [d] \setminus \{i\}} m_{ij} \sum_{k \geq 1} \beta_{\ell k}\beta_{jk} = 0
         ,\;\; \forall  i \neq \ell, \, i,\ell \in [d] \ .
    \end{align}
    Using the constraint that $\sum_{k \geq 1} \beta_{\ell k} \beta_{jk} = 0$ for all $\ell \neq j$, the second sum in \Cref{eqn:kkt-condition2} reduces to $m_{i \ell} \sum_{k\geq 1} \beta_{\ell k}^2 = m_{i \ell} \mu_\ell$. Then, we have 
    \begin{equation}
        -2\mu_i \sum_{k \geq 1} \frac{\beta_{ik} \beta_{\ell k}}{\lambda_k}  = m_{i \ell} \sum_{k\geq 1} \beta_{\ell k}^2 
        = m_{i \ell} \mu_\ell
        ,\quad  \forall  i \neq \ell, \, i, \ell \in [d].
    \end{equation}
    Repeating the derivation by reversing $i$ and $\ell$, we also have
    \begin{equation}
        -2\mu_\ell \sum_{k \geq 1} \frac{\beta_{ik} \beta_{\ell k}}{\lambda_k} 
        = m_{\ell i} \mu_i = m_{i \ell} \mu_i,
    \end{equation}
    where $m_{\ell i}  =  m_{i \ell}$ is by definition of $m_{ij}$. Note that $\mu_i =  \sum_{k\geq 1} \beta_{ik}^2> 0$. Then, if $m_{i \ell}$ is zero, the sum $\sum_{k \geq 1} \frac{\beta_{ik} \beta_{\ell k}}{\lambda_k}$ also needs to be zero. Vice versa, if this sum is zero, then we must have $m_{i \ell}=0$. If both $m_{i \ell}$ and the sum term are nonzero, dividing the two equations side by side, we get $\mu_i / \mu_\ell = \mu_\ell / \mu_i$, which gives $\mu_i=\mu_\ell$, in which case we have $m_{i \ell} = -2\sum_{k\geq 1}\frac{\beta_{ik}\beta_{lk}}{\lambda_k} \neq 0$. Therefore, for each $i \neq \ell$, we have    
    \begin{align}\label{eq:kkt-condition3}
        m_{i \ell} = &-2 \sum_{k \geq 1} \frac{\beta_{ik} \beta_{\ell k}}{\lambda_k} \quad \mbox{and}\\
        m_{i \ell}= 0 \quad &\mbox{or} \quad \Big\{\mu_i = \mu_\ell \text{ and }  m_{i \ell} \neq 0 \Big\}.
    \end{align}
    This result implies that the parameters $(\mu_i)_{i \in [d]}$ can be partitioned into $n$ distinct groups $S_1, \dots, S_n$, such that $\mu_i = \mu_j$ for all $i,j \in S_k$, and $\mu_i \neq \mu_j$ whenever $i$ and $j$ belong to different groups. 
    Without loss of generality, we index these parameters by
    \begin{equation}
       \mu_1 = \cdots = \mu_{r_1}, \quad 
        \mu_{r_1+1} = \cdots = \mu_{r_2}, \quad     \dots, \quad 
        \mu_{r_{n-1}+1} = \cdots = \mu_{r_n},   
    \end{equation}
    for some integers $1 = r_0 \leq r_1 < r_2 < \cdots < r_n = d$.
    We denote by $G(i)$ the group index such that $r_{G(i)} < i \leq r_{G(i)+1}$. 
    By \Cref{eq:kkt-condition3}, we have $m_{ij}=0$ if $i,j$ are in different groups, then \Cref{eqn:kkt-condition} becomes
        \begin{equation} 
        \frac{\partial \mcL}{\partial \beta_{ik}} = 2\beta_{ik} \left(1 - \frac{\mu_{r_{G(i)}}}{\lambda_k} \right) - \sum_{j  \in [r_{G(i)}+1, r_{G(i)+1}] \setminus i} m_{ij}\beta_{jk} = 0 \ ,\;\; \forall  i \in [d], \; k \geq 1.
    \end{equation}
    Now, we define $\bM \in \R^{d \times d}$ with the entries $\bM_{ij} =  \sum_{k \geq 1} \frac{\beta_{ik} \beta_{jk}}{\lambda_k} = -\frac{m_{ij}}{2} $, and $\bM_{ii} = \sum_{k \geq 1} \frac{\beta_{ik}^2}{\lambda_k} =1 $.
    The above equation becomes
    \begin{equation} 
        \beta_{ik} \left(\bM_{ii} - \frac{\mu_{r_{G(i)}}}{\lambda_k} \right) + \sum_{j  \in [r_{G(i)}+1, r_{G(i)+1}] \setminus i} \frac{-m_{ij}}{2}\beta_{jk} = 0 \ ,\;\; \forall  i \in [d], \; k \geq 1.
    \end{equation}
    Thus, for each $k \geq 1$, we have the following system of linear equations: 
    \begin{equation}
        \begin{bmatrix}
        \bM_1 -  \frac{\mu_{r_1}}{\lambda_k}\bI_1 & 0   & \cdots & 0 \\
        0   & \bM_2 -  \frac{\mu_{r_2}}{\lambda_k} \bI_2 & \cdots & 0 \\
        \vdots & \vdots & \ddots & \vdots \\
        0   & 0   & \cdots & \bM_n -  \frac{\mu_{r_n}}{\lambda_k} \bI_n
        \end{bmatrix}
        \begin{bmatrix}
        \beta_{1k} \\ 
        \beta_{2k} \\
        \cdots \\
        \beta_{dk} \\
        \end{bmatrix} = \mathbf{0}, \; \; \forall k \geq 1,
    \end{equation}
    where $\bM_i \in \R^{|S_i|\times |S_i|}$ is the submatrix in $\bM$ corresponding to group $S_i$. 
    For each $1 \leq i \leq n$, we have 
    \begin{equation}
    \bM_i 
    \begin{bmatrix}
    \beta_{(r_{i-1}+1)k} \\ 
    \beta_{(r_{i-1}+2)k} \\
    \vdots \\
    \beta_{r_{i}k} \\
    \end{bmatrix} 
    = \frac{\mu_{r_i}}{\lambda_k}
    \begin{bmatrix}
    \beta_{(r_{i-1}+1)k} \\ 
    \beta_{(r_{i-1}+2)k} \\
    \vdots \\
    \beta_{r_{i}k} \\
    \end{bmatrix} \ ,   \; \; \forall k \geq 1.
    \end{equation}
    This yields that $\bbeta_{S_i k} \defeq [\beta_{(r_{i-1}+1)k} ;\beta_{(r_{i-1}+1)k}; \cdots; \beta_{r_{i}k} ] \in \R^{|S_i| \times 1}$ are either zero vectors or (unnormalized) eigenvectors of $\bM_i$, multiplied by some scalar $c_k \in \R$, and $(\frac{\mu_{r_i}}{\lambda_k})_{k\geq 1}$ are the eigenvalues. We note that by construction of the matrix $\bM$,
    \begin{equation}
    \bM_i = \begin{bmatrix}\bbeta_{S_i 1}; \bbeta_{S_i 2},\cdots, \bbeta_{S_i k},\cdots   \end{bmatrix} \begin{bmatrix}
        \frac{1}{\lambda_1} & 0   & \cdots & 0 & \cdots \\
        0   &\frac{1}{\lambda_2}& \cdots & 0 & \cdots \\
        \vdots & \vdots & \ddots & \vdots & \vdots \\
        0   & 0   & \cdots & \frac{1}{\lambda_k} & \cdots \\
         \vdots   &  \vdots  &  \vdots &  \vdots & \ddots \\
        \end{bmatrix}
        \begin{bmatrix}\bbeta_{S_i 1}^\top \\ \bbeta_{S_i 2}^\top \\
        \vdots \\
        \bbeta_{S_i k}^\top \\
        \vdots \\
        \end{bmatrix}.
    \end{equation}
    Therefore, the matrix $\bM_i$ is a full-rank positive definite matrix since the rows of the first matrix are orthogonal by constraints (\Cref{eq:rq-direct-eigenfunctions-appendix-objective-2}). 
    Next, since all $(\mu_{r_i} /\lambda_k)_{k \geq 1}$ are distinct and $\bM_i$ have at most $|S_i|$ eigenvalues, we must have at most $|S_i|$ nonzero $\bbeta_{S_i k}$ vectors. 
    We then construct a matrix using these nonzero vectors as columns:
    $[\bbeta_{S_i k}]_{k \in \mcK}$ form a $|S_i| \times |\mcK|$ matrix with orthogonal rows (due to \Cref{eq:rq-direct-eigenfunctions-appendix-objective-2} and the fact that all other entries in the row are zero). Therefore, it should have exact $|S_i|$ rank and implies that $|\mcK| = |S_i|$. As all $(\mu_{r_i} /\lambda_k)_{k \in \mcK}$ are distinct, we can then conclude that all columns are corresponding to different eigenvectors of $\bM_i$. 
    We index such $k \in \mcK$ as $(k_j)_{r_{i-1}< j \leq r_i}$.
    We then conclude that the eigenvalues of $\bM_i$ are $(\frac{\mu_{r_i}}{\lambda_{k_j}})_{r_{i-1}< j \leq r_i}$.

    Since the sum of eigenvalues of $\bM_i$ equal to the trace of $\bM_i$, we have 
    \begin{equation}
    \trace(\bM_i) = r_i - r_{i-1} 
    =\sum_{j=r_{i-1}+1}^{r_i} \frac{\mu_{r_i}}{\lambda_{k_j}} 
    \Rightarrow \; \; \mu_{r_{i-1} + 1} = \mu_{r_{i-1}+2} = \cdots = \mu_{r_i} = \frac{r_i - r_{i-1}}{\sum_{j=r_{i-1}+1}^{r_i} \frac{1}{\lambda_{k_j}}}.
    \end{equation}
    By the arithmetic-harmonic mean inequality, we have 
    \begin{align}\label{eq:beta-upper-bound-1}
        \sum_{j=r_{i-1}+1}^{r_i}  \mu_j = \frac{(r_i - r_{i-1})^2}{\sum_{j=r_{i-1}+1}^{r_i} \frac{1}{\lambda_{k_j}}} \leq \sum_{j=r_{i-1}+1}^{r_i}  \lambda_{k_j}, \quad \forall i \in [d] 
    \end{align}
    At this juncture, also note that all $(k_i)_{i \in [d]}$ are distinct. Specifically, if $i$ and $j$ belong to the same group, then $\lambda_{k_i}$ and $\lambda_{k_j}$ being distinct, we have $k_i \neq k_j$; if $i$ and $j$ belong to different groups, then
    \begin{equation}
    \frac{\mu_{r_{G(i)}}}{\lambda_{k_i}} \neq \frac{\mu_{r_{G(j)}}}{\lambda_{k_j}} \text{ since $r_{G(i)}$ are different across groups.} 
    \end{equation}
    In this case, the vector 
    $[0;\, \cdots;\, \bbeta_{S_i k_i};\, \cdots;\, \bbeta_{S_j k_j};\, \cdots;\, 0] \in \mathbb{R}^{d}$
    is not an eigenvector of the block-diagonal matrix $\bM$, which leads to a contradiction since the eigenvectors of $\bM$ must arise directly from the (zero-padded) eigenvectors of its diagonal blocks. Therefore, summing \Cref{eq:beta-upper-bound-1} for $i\in[d]$, we have
    \begin{equation}
        \sum_{i=1}^{d}\sum_{k \geq 1} \beta_{ik}^2 = \sum_{i=1}^{d} \mu_i \leq \sum_{i=1}^{d} \lambda_{k_i} \leq \sum_{i=1}^{d} \lambda_i \ .    
    \end{equation}
    The equality holds if and only if the arithmetic-harmonic mean inequality in \Cref{eq:beta-upper-bound-1} holds with equality for each $i$, which implies $\lambda_{k_j} = \lambda_{k_\ell}$ for $j,\ell \in [r_{i-1}+1,r_i]$. Since all $(k_i)_{i \in [d]}$ are distinct, this implies $r_i = i$ for $i \in [d]$, \ie each group $S_i$ consists of only one element, and we have $\mu_i=\lambda_{k_i}$ for all $i\in[d]$. Now, \Cref{eq:kkt-condition3} implies that $m_{i \ell}=0$ for all $i \neq \ell$, since $\mu_i \neq \mu_{\ell}$.   Substituting $m_{i \ell}=0$ into \Cref{eqn:kkt-condition}, we obtain
    \begin{align} \label{eqn:kkt-condition4}
      \beta_{ik}\left(1 -\frac{\mu_i}{\lambda_k} \right) = 0 \ , \quad \forall i, k \in [d] \ .
    \end{align}
    However, since $(\lambda_k)_{k \in [d]}$ and $(\mu_i)_{i \in [d]}$ are distinct, for a given $i\in[d]$, $\beta_{ik}$ can be nonzero for at most one value of $k \in [d]$. Moreover, all $(\beta_{ik})_{k \in [d]}$ cannot be zero due to the constraints. Therefore, there exists a surjective mapping $\sigma : [d] \to [d]$ such that 
    \begin{equation}
        \beta_{ik} = 0 \ , \;\; \forall k \neq \sigma(i) \ , \quad \mbox{and} \quad \beta_{i \sigma(i)} = \pm \sqrt{\lambda_{\sigma(i)}} \;\; \implies \alpha_{i k} = \pm \mathds{1}\{k = \sigma(i)\} \ .
    \end{equation}
    This implies that  
    \begin{equation}
        0 = -2\sum_{k=1}^{d} \frac{\beta_{ik}\beta_{\ell k}}{\lambda_k} =  -2 \cdot \mathds{1}\{\sigma(i) = \sigma(\ell)\} \ , \quad \forall i \neq j, \, i,j \in [d] \ ,
    \end{equation}
    and consequently, $\sigma: [d] \mapsto [d]$ is a permutation. To sum up, the maximum value of the constrained problem is $\sum_{i=1}^{d} \lambda_i$, and substituting $\balpha_i$ parameterization, the optimal solutions satisfy $\hat\phi_i = \pm \phi^*_{\sigma(i)}$ for some permutation $\sigma: [d] \mapsto [d]$, which completes the proof.
\end{proof}

\subsection{Connection to Related Work on Rayleigh Quotient Optimization}\label{app:related-rayleigh}

Spectral Inference Networks (SpIN) proposed by \cite{pfau2019spectral} is the first parametric framework to learn parameterized eigenfunctions. The core idea of SpIN stems from the simple yet elegant observation that we can write the Rayleigh quotient optimization problem for extracting the top-$d$ eigenspace of a kernel in a similar way as extracting the top-$d$ eigenspace of a matrix. For $\Phi: \px \to \R^d$, define 
\begin{equation}
    \bSigma_\Phi : [\bSigma_\Phi]_{i,j}= \braket{\phi_i}{\phi_j} \ , \quad \text{and} \quad \bPi_{\Phi} : [\bPi_{\Phi}]_{i,j} = \braket{\phi_i}{T \phi_j} \ .   
\end{equation}
Then, we can write the following constrained optimization objective for extracting the top-$d$ eigenspace: 
\begin{equation}\label{eq:spin-joint-constrained}
    \max_{\Phi} \trace(\bPi_{\Phi}) = \max_{\Phi} \braket{\Phi(X)}{T \Phi(X)} \ , \quad
    \st \quad \braket{\phi_i}{\phi_j} = \delta_{ij} \ . 
\end{equation}
Since the optimization problem is constrained, subsequent work has explored various approaches to address it. \citet{pfau2019spectral} proposed a bilevel optimization algorithm to enable online learning, while \citet{wu2023dive} improved upon this by introducing a novel regularization mechanism. \citet{deng2022neural} further modified the objective in \Cref{eq:spin-joint-constrained} by altering the constraints as
\begin{equation} \label{eqn:neuralEF_single}
    \max_{\phi_i} \braket{\phi_i}{T \phi_i} \quad \st \;\; \braket{\phi_i}{\phi_i} = 1 \ , \quad \braket{\phi_i}{T \phi_j}=0 \ , \;\; \forall j \in [i-1] \ ,
\end{equation}
and use EigenGame~\citep{gemp2021eigengame,gemp2022eigengame} perspective for relaxing the constraints. 
We note that this objective alone can only solve one eigenfunction $\phi_i$. To enable exact eigenfunction recovery for all top-$d$ eigenfunctions, they employ stop gradient tricks (see \Cref{sec:sequential-to-joint}) to enable joint optimization, which, however, violates the joint parameterization desiderata.
In contrast, our \Cref{thm:rq-direct-eigenfunctions} provides a more general formulation: by directly maximizing the sum over  \Cref{eqn:neuralEF_single} over all $1 \le i \le d$, we can recover the exact eigenfunctions (up to permutation) without relying on additional mechanisms such as stop gradients.  
Additionally, we emphasize that our \Cref{thm:rq-eigenspace-extractor} applies to non-self-adjoint operators, as it is capable of recovering singular functions, unlike the methods proposed in prior work \citep{deng2022neural,pfau2019spectral}.

\section{Detailed Algorithms for Rayleigh-Ritz Method} \label{app:rayleigh_ritz_implementation}
In this section, we present a detailed implementation of the Rayleigh–Ritz methods for both the LoRA and RQ objectives within our framework. Since the positive-pair kernel $\kaa(a,a') = \frac{\pplus(a,a')}{\pa(a)\pa(a')}$ is defined only implicitly, and modern learning algorithms typically have only stochastic access to the underlying data distribution, a streaming algorithm is required to align with contemporary practical settings.

\paragraph{Rayleigh quotient.}
We begin with the Rayleigh–Ritz algorithm for the Rayleigh quotient objective (\Cref{eq:rq-population}), which produces an orthonormal set of basis functions ${\hat{\psi}_1, \ldots, \hat{\psi}_d}$ due to its normalization constraints. Consequently, \Cref{thm:rayleigh-ritz} applies directly.

For the positive-pair kernel, the empirical covariance matrix $\bB \in \mathbb{R}^{d \times d}$ has entries 
\begin{equation}
    \bB_{ij} \vcentcolon = \braket{\hat \psi_i}{\taa \hat \psi_j} = \E_{x \sim \px} \E_{a,a^+ \sim p^+(\cdot|x)}[\hat \psi_i(a) \hat \psi_j(a^+)] \ ,
\end{equation}
which can be estimated from training samples. The detailed procedure is provided in \Cref{alg:rr_rq}. 

In contrast, for the VICReg objective (\Cref{eqn:vicreg}), the encoder outputs are not centered. To handle this, we employ Welford’s algorithm to estimate the covariance matrix in a single pass under an online (streaming) setting. The detailed procedure is provided in \Cref{alg:rr_vicreg}.

\begin{algorithm}[h!]
    \caption{The Rayleigh–Ritz method for representations learned from the RQ objective (\Cref{eq:rq-loss}). }
    \label{alg:rr_rq}
    \begin{algorithmic}[1]
    \Require A trained encoder $\hat{\Psi} : \mathcal{A} \to \mathbb{R}^d$ with RQ , access to $p^+(\cdot | x)$.
    
    \State \textbf{(Training Phase)}
    \State Initialize the covariance matrix $\bB \in \mathbb{R}^{d \times d}$ as a zero matrix, and the counter variable $\texttt{cnt} = 0$.
    \For{each step}
        \State Sample a batch of samples $\{x_1, \cdots, x_m\}$.
        \State Sample positive pairs $(a_i, a_i^+) \sim p^+(\cdot | x_i)$ for $i \in [m]$.
        \State Update the covariance matrix:
        \[
            \bB \leftarrow \frac{1}{\texttt{cnt} + m}
         \left(
             \texttt{cnt} \cdot \bB + \sum_{i=1}^m  \hat{\Psi}(a_i)\hat{\Psi}(a_i^+)^\top 
        \right).
        \]
    \State Update the counter variable: $\texttt{cnt} \leftarrow \texttt{cnt} + m$.
\EndFor
\State Perform eigendecomposition $\bB = \bU^\top \mathbf{\Sigma} \bU$, where the eigenvalues in $\mathbf{\Sigma}$ are sorted in descending order.

\State \textbf{(Testing Phase)}
\State Given a test sample $a'$, eigenfunctions are given as $\Psi^*(a') =  \hat \Psi(a') \bU $.
\end{algorithmic}
\end{algorithm}

\begin{algorithm}[h!]
    \caption{The Rayleigh–Ritz method for representations learned from the VICReg objective (\Cref{eqn:vicreg}). }
    \label{alg:rr_vicreg}
    \begin{algorithmic}[1]
    \Require A trained encoder $\hat{\Psi} : \mathcal{A} \to \mathbb{R}^d$ with VICReg , access to $p^+(\cdot | x)$.
    
    \State \textbf{(Training Phase)}
    \State Initialize the mean vector $\mu \in \mathbb{R}^d$ as a zero vector, the covariance matrix $\bB \in \mathbb{R}^{d \times d}$ as a zero matrix, and the counter variable $\texttt{cnt} = 0$.
    \For{each step}
        \State Sample a batch of samples $\{x_1, \cdots, x_m\}$.
        \State Sample positive pairs $(a_i, a_i^+) \sim p^+(\cdot | x_i)$ for $i \in [m]$.
        \State Let $\delta_i = \Psi(a_i) - \mu $ for $i \in [m]$.
        \State Update the mean:
    \[
        \mu \leftarrow \frac{1}{\texttt{cnt} + m}
        \left(
            \texttt{cnt} \cdot \mu + \frac{1}{2} \sum_{i=1}^m \left( \hat{\Psi}(a_i) + \hat{\Psi}(a_i^+) \right)
        \right).
    \]
    \State Update the covariance matrix:
    \[
        \bB \leftarrow \frac{1}{\texttt{cnt} + m}
        \left(
            \texttt{cnt} \cdot \bB + \sum_{i=1}^m \delta_i \left( \hat{\Psi}(a_i^+) - \mu \right)^\top
        \right).
    \]
    \State Update the counter variable: $\texttt{cnt} \leftarrow \texttt{cnt} + m$.
\EndFor
\State Perform eigendecomposition $\bB = \bU^\top \mathbf{\Sigma} \bU$, where the eigenvalues in $\mathbf{\Sigma}$ are sorted in descending order.

\State \textbf{(Testing Phase)}
\State Given a test sample $a'$, eigenfunctions are given as $\Psi^*(a') = (\hat \Psi(a') - \mu) \bU $.
\end{algorithmic}
\end{algorithm}

\paragraph{Low-rank approximation.}
For low-rank approximation objectives, the extracted basis functions are not orthonormal since its optimal solution is given by $\hat \Psi = \bQ \sqrt{\mathbf{\Sigma}} \Psi^*_d$, where $\bQ \in \R^{d \times d}$ is an orthonormal matrix, $\sqrt{\bSigma} \in \R^{d \times d}$ is a diagonal matrix containing the square root of the top-$d$ eigenvalues in descending order, and $\Psi^*_d = [\psi^*_1,\dots,\psi^*_d]$ are the top-$d$ eigenfunctions. Therefore, we can directly compute the matrix $\bB_{ij} \vcentcolon = \braket{\hat \psi_i}{\hat \psi_j}$ and perform eigendecomposition on the matrix $\bB$. The detailed procedure for spectral contrastive loss is provided in \Cref{alg:rr_scl}. 

\begin{algorithm}[h!]
    \caption{The Rayleigh–Ritz method for representations learned from the SCL objective (\Cref{eq:spectral-contrastive-loss}). }
    \label{alg:rr_scl}
    \begin{algorithmic}[1]
    \Require A trained encoder $\hat{\Psi} : \mathcal{A} \to \mathbb{R}^d$ with SCL , access to $p^+(\cdot | x)$.
    
    \State \textbf{(Training Phase)}
    \State Initialize the covariance matrix $\bB \in \mathbb{R}^{d \times d}$ as a zero matrix, and the counter variable $\texttt{cnt} = 0$.
    \For{each step}
        \State Sample a batch of samples $\{x_1, \cdots, x_m\}$.
        \State Sample positive pairs $(a_i, a_i^+) \sim p^+(\cdot | x_i)$ for $i \in [m]$.
        \State Update the covariance matrix:
        \[
            \bB \leftarrow \frac{1}{\texttt{cnt} + 2m}
         \left(
             \texttt{cnt} \cdot \bB + \sum_{i=1}^m  \hat{\Psi}(a_i)\hat{\Psi}(a_i)^\top + \sum_{i=1}^m  \hat{\Psi}(a_i^+)\hat{\Psi}(a_i^+)^\top
        \right).
        \]
    \State Update the counter variable: $\texttt{cnt} \leftarrow \texttt{cnt} + 2m$.
\EndFor
\State Perform eigendecomposition $\bB = \bU^\top \mathbf{\Sigma} \bU$, where the eigenvalues in $\mathbf{\Sigma}$ are sorted in descending order.

\State \textbf{(Testing Phase)}
\State Given a test sample $a'$, eigenfunctions are given as $\Psi^*(a') = \mathbf{\Sigma}^{-1/2} \hat \Psi(a') \bU $.
\end{algorithmic}
\end{algorithm}

\section{Experimental Details and Additional Results}\label{app:exp-details}
We provide additional experimental results and details for both synthetic (\Cref{app:synthetic-exp}) and real-world data experiments (\Cref{app:image-exp}). \looseness=-1 

\subsection{Experiments on Synthetic Data} \label{app:synthetic-exp}

In this section, we provide a detailed construction of the synthetic data kernel, including both the theoretical validation that establishes it as a valid kernel and the rejection sampling algorithm used for its generation. In addition, \Cref{fig:eigenfunction_comparison_appendix} offers a visual comparison of eigenfunctions estimated by different methods. 

\paragraph{Construction of synthetic data kernel.}
Our synthetic data kernels, $\kaa(a,a') = \frac{ P^+(a,a')}{\pa(a)\pa(a')}$, are constructed using well-established orthogonal function bases, such as Legendre polynomials and Fourier base functions. Formally, given the input dimensionality $p$, and $r$ positive eigenvalues $1=\lambda_1 \geq \lambda_2 \cdots \geq \lambda_{r} > 0$, we let $\kaa(a,a') = \sum_{i=1}^{r} \lambda_i \psi_i(a) \psi_i(a')$, where 
\begin{equation}
\psi_i(a) = 
\begin{cases}
\prod_{j=1}^p  \sqrt{2i-1}P^{L}_{i-1}(a_j) & \text{ for Legendre kernels}, \\
\prod_{j=1}^p \sqrt{2} \cos((i-1)\pi a_j), & \text{ for Fourier kernels},
\end{cases}
\end{equation}
where $P^{L}_i$ denotes the $i^{th}$ Legendre polynomial.
The scalars $\sqrt{2i-1}$ and $\sqrt{2}$ are added to make the functions orthonormal under $\lap$.
For both kernels, the corresponding joint distribution and marginal distributions are defined as follows:
\begin{equation}
\forall a, a' \in [-1,1]^p, \;\; \pa(a)=  \pa(a') = \frac{1}{2^p}, \;\; \text{and} \;\; P^+(a,a') =  \frac{1}{2^{2p}} \sum_{i=1}^{r} \lambda_i \psi_i(a) \psi_i(a') \ .
\end{equation}

To ensure that $P^+(a,a')$ is a valid joint distribution, we prove the following sufficient condition. 
\begin{proposition}
\label{prop:eigenvalues_synthetic}
$P^+(a,a')$ is a valid joint distribution if
$\sum_{i=2}^{r} (2i-1)^p\lambda_i \leq 1$ for Legendre kernels and $\sum_{i=2}^{r} 2^p\lambda_i \leq 1$ for Fourier kernels. 
\end{proposition}
\begin{proof}
We first show that $\int_{x \in [-1,1]^p} P^+(a,a') da = \pa(a')$ for both kernels:
\begin{align}
    \int_x P^+(a,a') da &= \frac{1}{2^{2p}} \sum_{i=1}^{r} \lambda_i \psi_i(a')\int_{a \in [-1,1]^p} \psi_i(a) da  \\
    &= \frac{1}{2^{2p}} \lambda_1 \psi_1(x')\int_{a \in [-1,1]^p} \psi_1(a) da \\
    &= \frac{1}{2^p} = \pa(a),
\end{align}
where we use the fact that $\int_{-1}^1 P^{L}_i(a_j) da_j = \int_{-1}^1 \cos(i\pi a_j) da_j = 0$ for all $i \geq 1$ and $2 \leq j \leq p$.  

Next, we show that $P^+(a,a') \geq 0$ for all $a, a' \in [-1,1]^p$. 
$$
P^+(a,a') 
= \frac{1}{2^{2p}} \sum_{i=1}^{r} \lambda_i \psi_i(a) \psi_i(a')
= \frac{1}{2^{2p}} +  \frac{1}{2^{2p}} \sum_{i=2}^{r} \lambda_i \psi_i(a) \psi_i(a')
\geq \frac{1}{2^{2p}}  - \frac{1}{2^{2p}} \sum_{i=2}^{r} \lambda_i |\psi_i(a)| |\psi_i(a')|.
$$
Since we have $|P^L_i(a_j)| \leq 1$ and $|\cos(i \pi a_j)| \leq 1$ for all $i$ and $j$, we have 
$|\psi_i(a)| \leq (2i-1)^{p/2}$ for Legendre kernels and $|\psi_i(a)| \leq 2^{p/2}$ for Fourier kernels. Plug these inequalities to the former one, we get $P^+(a,a')  \geq 0$.
Therefore, $P^+(a,a') $ is a valid distribution.
\end{proof}

\paragraph{Training details.}
For both Legendre and Fourier kernels, we define the rank of the synthetic kernel as $r \in \{6, 8, 10\}$, with eigenvalues given by an exponential decay $\lambda_i = c e^{-0.3 i}$, where the constant $c$ is the largest value satisfying Proposition~\ref{prop:eigenvalues_synthetic}. The encoder output dimension is set to $d = r/2$, extracting the top half of the eigenfunctions.

We randomly generate $10^7$ training samples from the joint $P^+(a,a')$. The encoder is a multi-layer perceptron with four hidden layers, each containing $128$ neurons and GeLU activation functions. Training is performed using the Adam optimizer with a learning rate of $10^{-3}$ and a minibatch size of $1,000$. The total number of training steps is $3 \times 10^5$. To facilitate the training process,  we append the input $x$ with polynomial features $\{a_1^{i_1}a_2^{i_2}\cdots a_p^{i_p}\}_{0 \leq i_1,\cdots,i_p \leq r}$  for Legendre kernels,  and Fourier features $\{\cos(i \pi a_j)\}_{0 \leq i \leq r, 1\leq j \leq p}$ for Fourier kernels.
This setting is fixed across all methods. 

For evaluation, we generate another $10^6$ samples independently to estimate the L2 distance. For RQ and VICReg objectives, we set the penalty hyperparameters $\mu=10$ and $\nu=30$.

For the joint nesting method, we simply take the average of all objectives across all dimensions. For Rayleigh-Ritz methods, we adopt \Cref{alg:rr_scl} for LoRA, \Cref{alg:rr_rq} for RQ, and \Cref{alg:rr_vicreg} for VICReg, using all training samples. 

\paragraph{Visualization of the estimated eigenfunctions.}

\Cref{fig:eigenfunction_comparison_appendix} illustrates the eigenfunctions in a representative setting where the true kernel has rank $r=8$ and we extract $d=4$ eigenfunctions.

\begin{figure}[t!]
    \centering
    \begin{subfigure}[b]{\textwidth}
        \includegraphics[width=\textwidth]{figures/synthetic_exp/legendre_Rayleigh_Ritz_eigenfunction.png}
        \caption{Legendre kernels with Rayleigh-Ritz post-processing ($p=1, r=8, d=4$).}
        \label{fig:legendre_rr}
    \end{subfigure}
    
    \begin{subfigure}[b]{\textwidth}
        \includegraphics[width=\textwidth]{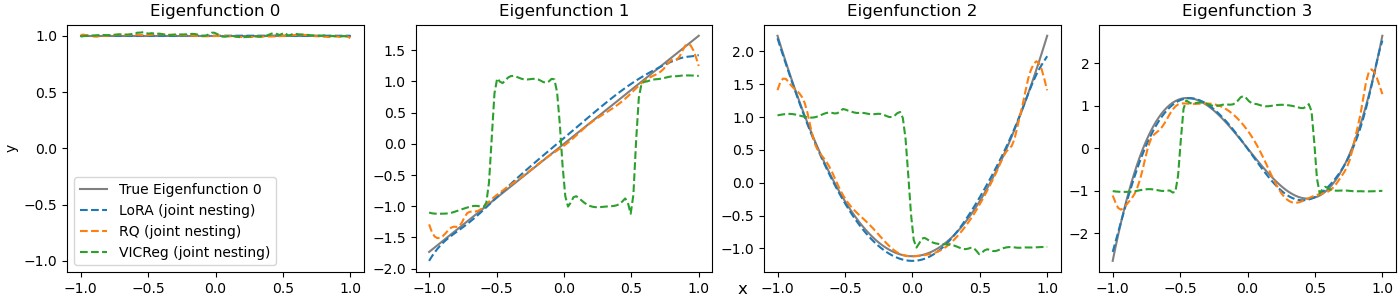}
        \caption{Legendre kernels with joint nesting ($p=1, r=8, d=4$).}
        \label{fig:legendre_jn}
    \end{subfigure}
    
        \vspace{0.5cm}

    \begin{subfigure}[b]{\textwidth}
        \includegraphics[width=\textwidth]{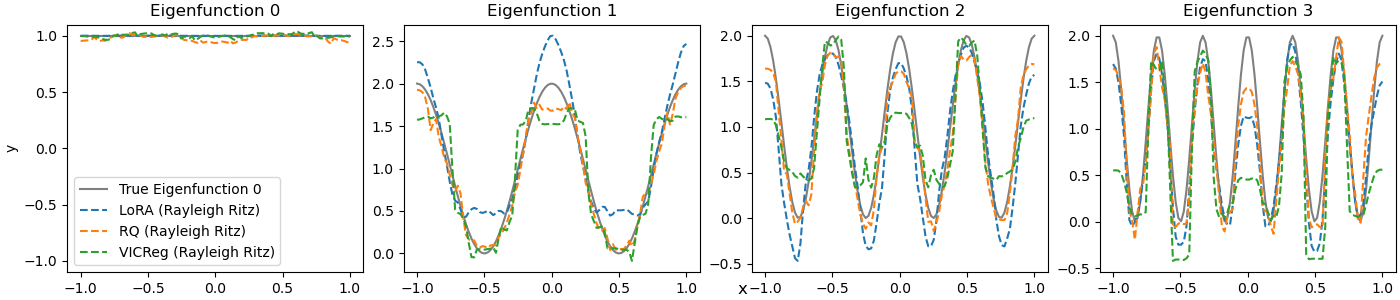}
        \caption{Fourier kernels with Rayleigh-Ritz post-processing ($p=2, r=8, d=4$).}
        \label{fig:fourier_rr}
    \end{subfigure}

    \begin{subfigure}[b]{\textwidth}
        \includegraphics[width=\textwidth]{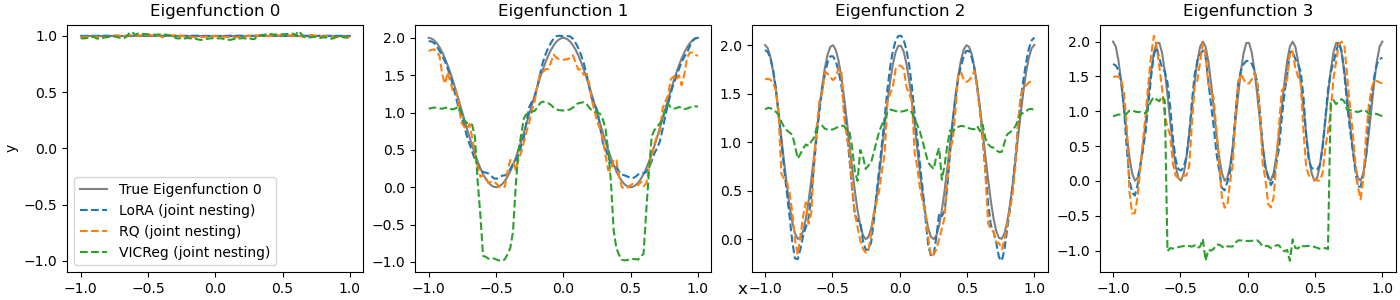}
        \caption{Fourier kernels with joint nesting ($p=2, r=8, d=4$).}
        \label{fig:fourier_jn}
    \end{subfigure}

    \caption{A comparison of eigenfunctions estimated by different methods (LoRA, RQ, and VICReg) against the true eigenfunctions (solid gray line). Each subplot shows the results for a specific kernel type and post-processing technique with parameters. For an input dimension $p=2$, the plot shows the functions $\psi_i(a,a)$ from $a=-1$ to $a=1$. }
    \label{fig:eigenfunction_comparison_appendix}
\end{figure}

\subsection{Experimental Details of the Adaptive-dimensional Representations}\label{app:image-exp}

\paragraph{Datasets.}
We use the following datesets for our experiments. We report the validation performance on ImageNette as the test sets are not available.  
\begin{enumerate}

    \item \textbf{CIFAR-10}~\citep{krizhevsky2009learning} is a widely-used image classification benchmark consisting of $60,000$ $32 \times 32$ pixel color images across $10$ object classes. The dataset is balanced, with $50,000$ images for training and $10,000$ for testing.
    
    \item \textbf{ImageNette}~\footnote{\url{https://github.com/fastai/imagenette}} is a subset of ImageNet designed for rapid experimentation. It consists of 10 classes (tench, English springer, cassette player, chain saw, church, French horn, garbage truck, gas pump, golf ball, parachute). The dataset contains $9,469$ training images and $3,925$ validation images with variable resolutions, which are typically resized during preprocessing.

\end{enumerate}

\paragraph{Data Augmentation.}
For ImageNette, we adopt the image augmentation settings used in VICReg~\citep{bardes2022vicreg}, and for CIFAR-10, we follow those from SimSiam~\citep{chen2021exploring}. The applied augmentations include random cropping, horizontal flipping, color jittering, grayscale conversion, Gaussian blur, solarization, and color normalization.

\paragraph{Network architectures.}
We employ ResNet-18~\citep{he2016deep} as the backbone network, followed by a two-layer MLP with a hidden dimension of 2048 and an output linear layer chosen from adaptive dimensions of lengths $r=\{4,8,16,32,64,128,256\}$.

\paragraph{Optimization details.}
We train our models on the ImageNette/CIFAR-10 dataset for $600/800$ epochs using the Stochastic Gradient Descent (SGD) optimizer. The optimizer is configured with a momentum of $0.9$, a weight decay of $5 \times 10^{-4}$, and an initial learning rate of $0.1$. The learning rate is dynamically adjusted throughout training using a cosine decay schedule~\citep{loshchilov2016sgdr} with $60$ warmup epochs. 
To prevent model collapse and ensure training stability, we employ gradient clipping with a maximum norm of $3.0$.  
Our implementation is built using PyTorch and the lightly library~\citep{susmelj2020lightly}. All models are trained on a single NVIDIA A6000 GPU.

\paragraph{Objectives.}
Given a mini-batch of paired augmented samples $(a_i,a_i^+)_{i=1}^n$, the VICReg objective is implemented as follows:
\begin{equation} \label{eqn:vicreg}
\mathcal{L}_{\rm vicreg}(\Psi) = 
\frac{\lambda}{n} \sum_{i=1}^{n} \| \tilde{\Psi}(a_i) - \tilde{\Psi}(a_i^+)\|_2^2 
\; + \; \frac{\mu}{d} \sum_{i=1}^{d} \mbox{ReLU}\left(1 - \sqrt{C[i,i] + \epsilon} \right) 
\; + \; \frac{\nu}{d(d-1)} \sum_{i \neq j} C[i,j]^2,
\end{equation}
where $\tilde{\Psi}(a) = \Psi(a) -\frac{1}{n}\sum_{i=1}^n \Psi(a_i)$ is a centered encoder,
$d$ is the dimension of output representations, $\epsilon=10^{-4}$ is a small positive constant for numerical stability, and $C \in \mathbb{R}^{d \times d}$ is the empirical covariance matrix with $C[i,j] = \frac{1}{n- 1} \sum_{i=1}^{n} \tilde{\psi}_i(a) \tilde{\psi}_j(a)$, and
$\lambda,\mu$ and $\nu$ are positive hyperparameters.  We set $\lambda=50$, $\mu=25$, and $\nu=512$ in our experiments.

For the spectral contrastive loss, we follow the original implementation~\citep{haochen2021provable}, applying $\ell_2$ normalization to the final representation before computing the loss (\Cref{eq:spectral-contrastive-loss}), \ie $\Psi'(a) = \Psi(a) / \|\Psi(a)\|_2$. This normalization substantially stabilizes training and improves performance (by about 10\% on ImageNette).

\paragraph{Feature Importance Calculation.}
For joint nesting, we average all objectives across the selected dimensions with uniform weights. For the Rayleigh–Ritz methods, we use \Cref{alg:rr_scl} for SCL and \Cref{alg:rr_vicreg} for VICReg, training for an additional 10 and 30 epochs on ImageNette and CIFAR-10, respectively. For random selection, we train 300 linear classifiers in parallel, each using a randomly chosen subset of $r$ features, and report the averaged accuracy. 

\paragraph{Evaluation.}
We follow the standard protocol of training a linear classifier on top of frozen ResNet-18 representations. The linear classifier is trained for $50$ epochs using SGD with a learning rate of $0.1$, momentum of $0.9$, and weight decay of $10^{-4}$. The training augmentation pipeline includes random cropping to $224 \times 224$, random horizontal flips, and color normalization. For validation, images from ImageNette are center-cropped to $224 \times 224$ and normalized, images from CIFAR-10 are simply normalzied. The precise accuracy values shown in \Cref{fig:image_results} are given in \Cref{tab:cifar10} and \Cref{tab:imagenette}.

\begin{table}[t]
\centering
\begin{tabular}{c|c|c|c|c|c|c|c|c}
\toprule
\multicolumn{9}{c}{VICReg~\citep{bardes2022vicreg} in \Cref{eqn:vicreg}} \\
\midrule \midrule
Methods $\backslash$ Dims & 4 & 8 & 16 & 32 & 64 & 128 & 256 & 512 \\
\midrule
Rayleigh Ritz (RR) & 66.31 & 76.05 & \textbf{80.69} & \textbf{82.20} & 82.91 & 84.04 & 84.07 & 84.06 \\
\midrule
Joint nesting (JN) & \textbf{69.17} & \textbf{77.62} & 80.37 & 82.49 & 83.72 & 83.94 & 83.99 & \textbf{84.16} \\
\midrule
Fixed features (FF)   & 68.21 & 72.65 & 77.42 & 79.66 & \textbf{83.74} & \textbf{85.40} & \textbf{84.37} & 84.10 \\
\midrule
Random selection (RS)   & 30.02 & 42.07 & 57.55 & 76.16 & 82.97 & 83.50 & 83.81 & 84.01 \\
\bottomrule
\multicolumn{9}{c}{SCL~\citep{haochen2021provable} in \Cref{eq:spectral-contrastive-loss} } \\
\midrule \midrule
Methods $\backslash$ Dims & 4 & 8 & 16 & 32 & 64 & 128 & 256 & 512 \\
\midrule
Rayleigh Ritz (RR) & 75.62 & 83.43 & 86.47 & 86.89 & 86.88 & 86.89 & 86.89 & 86.90\\
\midrule
Joint nesting (JN) & \textbf{77.55} & \textbf{83.89} & 86.38 & 86.64 & 86.66 & 86.67 & 86.66 & 86.66\\
\midrule
Fixed features (FF)  & 72.67 & 81.05 & \textbf{86.72} & \textbf{86.94} & \textbf{87.00} & \textbf{87.07} & \textbf{86.94} & \textbf{86.91} \\
\midrule
Random selection (RS) & 61.94 & 79.10 & 84.48 & 85.61 & 86.10 & 86.39 & 86.62 & 86.72\\
\midrule
\bottomrule
\end{tabular}
\caption{
Linear probe accuracies (\%) of different representation sizes on CIFAR-10.
}
\label{tab:cifar10}
\end{table}

\begin{table}[t]
\centering
\begin{tabular}{c|c|c|c|c|c|c|c|c}
\toprule
\multicolumn{9}{c}{VICReg~\citep{bardes2022vicreg} in \Cref{eqn:vicreg}} \\
\midrule \midrule
Methods $\backslash$ Dims & 4 & 8 & 16 & 32 & 64 & 128 & 256 & 512 \\
\midrule
Rayleigh Ritz (RR)  & 78.39 & \textbf{87.39} & \textbf{89.15} & \textbf{89.81} & \textbf{89.94} & \textbf{89.94} & \textbf{89.89} & \textbf{89.89} \\
\midrule
Joint nesting (JN) & \textbf{81.78} & 87.08 & 88.61 & 88.84 & 88.76 & 88.89 & 89.04 & 89.07 \\
\midrule
Fixed features (FF)  & 72.51 & 81.04 & 87.08 & 88.79 & 88.71 & 89.10 & 89.30 & 89.86\\
\midrule
Random selection (RS) & 51.72 & 75.76 & 87.98 & 89.15 & 89.54 & 89.68 & 89.67 & 89.65\\
\bottomrule
\multicolumn{9}{c}{SCL~\citep{haochen2021provable} in \Cref{eq:spectral-contrastive-loss} } \\
\midrule \midrule
Methods $\backslash$ Dims & 4 & 8 & 16 & 32 & 64 & 128 & 256 & 512 \\
\midrule
Rayleigh Ritz (RR) & 49.12 & 71.41 & 79.39 & 83.39 & 84.56 & 84.89 & 84.76 & 84.76 \\
\midrule
Joint nesting (JN) & \textbf{82.06} & \textbf{87.44} & \textbf{88.56} & \textbf{88.94} & \textbf{88.69} & \textbf{88.87} & \textbf{88.89} & \textbf{88.89} \\
\midrule
Fixed features (FF) & 71.26 & 86.22 & 87.92 & 88.10 & 87.75 & 86.11 & 86.34 & 84.79 \\
\midrule
Random selection (RS) &  32.23 & 69.19 & 81.93 & 84.25 & 84.61 & 84.74 & 84.73 & 84.79 \\
\midrule
\bottomrule
\end{tabular}
\caption{
Linear probe accuracies (\%) of different representation sizes on Imangenette.
}
\label{tab:imagenette}
\end{table}

\subsection{Comparison between RQ and VICReg}
\label{app:compare_rq_vicreg}
In this section, we compare the RQ objective (\Cref{eq:rq-loss}) and the VICReg objective (\Cref{eqn:vicreg}). Although their formulations appear similar, subtle differences in their design lead to substantial variations in empirical performance.

First, both methods compute the invariance term (the first term in \Cref{eq:rq-loss}) in the same way. However, VICReg applies a hinge loss penalty to the variance term (the second term), whereas RQ uses a squared loss penalty. The hinge loss enforces a lower bound on the variance, penalizing only when it falls below 1, thereby encouraging representations with unit or greater variance. Moreover, VICReg estimates this term using minibatch statistics, introducing bias in the estimation of the population variance. In contrast, RQ employs a squared penalty on the population variance and uses sample splitting to obtain an unbiased estimator.

Similarly, VICReg’s estimation of the covariance term is also biased. While both methods use a squared penalty, VICReg relies on the empirical covariance matrix computed within a minibatch, without sample splitting, resulting in a biased estimator of the population covariance.

In our synthetic experiments (\Cref{tab:synthetic_comparison}, \Cref{fig:eigenfunction_comparison_appendix}), VICReg yields poorer estimates of eigenfunctions and eigenvalues. Nonetheless, it achieves much higher linear probe accuracy on CIFAR-10 (84\% vs. 75\%). We attribute this improvement to the hinge loss, which prevents representation collapse, a phenomenon also noted in the original paper by \citet{bardes2022vicreg}, who observe that the hinge loss is critical to prevent variance gradient from being too small.

\bibliography{ref}

\end{document}